\documentclass{article}

\usepackage[utf8]{inputenc} 
\usepackage[T1]{fontenc}    
\usepackage[CJKbookmarks=true,
            bookmarksnumbered=true,
			bookmarksopen=true,
			colorlinks=true,
			citecolor=blue,
			linkcolor=blue,
			anchorcolor=red,
			urlcolor=blue]{hyperref}
\usepackage{url}            
\usepackage{booktabs}       
\usepackage{amsfonts}       
\usepackage{nicefrac}       
\usepackage{microtype}      
\usepackage{xcolor}         
\usepackage[nonatbib,final]{neurips_2024}
\usepackage{titlesec}
\usepackage{amsmath}
\usepackage{amsthm}
\usepackage{bm}
\usepackage{graphicx}
\usepackage{enumitem}
\usepackage{multirow}
\usepackage{graphicx}    
\usepackage{subcaption}  
\usepackage{booktabs}
\usepackage{float}
\usepackage[noend]{algpseudocode}
\usepackage[algorithm]{algorithm}
\usepackage{adjustbox}
\usepackage{wrapfig}
\usepackage{caption}

\usepackage{multicol}
\usepackage{enumitem}
\usepackage{parskip}

\newcommand*\Let[2]{\State #1 $\gets$ #2}
\algrenewcommand\algorithmicrequire{\textbf{Precondition:}}
\algrenewcommand\algorithmicensure{\textbf{Postcondition:}}

\newtheorem{lemma}{Lemma}
\newtheorem{theorem}{Theorem}
\newtheorem{corollary}{Corollary}
\newtheorem{definition}{Definition}[section]

\title{UniGAD: Unifying Multi-level Graph Anomaly Detection}

\author{Yiqing Lin$^{1}$\thanks{Work done as a visiting student at Hong Kong University of Science and Technology.}~, Jianheng Tang$^{2,3}$, Chenyi Zi$^{3}$, H.Vicky Zhao$^{1}$, Yuan Yao$^{2}$, Jia Li$^{2,3}$\thanks{Corresponding Author.}
  \\
  $^{1}$Tsinghua University \\
$^{2}$Hong Kong University of Science and Technology \\
$^{3}$Hong Kong University of Science and Technology (Guangzhou)\\
 \texttt{linyq20@mails.tsinghua.edu.cn, jtangbf@connect.ust.hk, }\\\texttt{barristanzi666@gmail.com, vzhao@tsinghua.edu.cn, \{yuany,jialee\}@ust.hk}}

\begin{document}

\maketitle

\begin{abstract}

Graph Anomaly Detection (GAD) aims to identify uncommon, deviated, or suspicious objects within graph-structured data. Existing methods generally focus on a single graph object type (node, edge, graph, etc.) and often overlook the inherent connections among different object types of graph anomalies. For instance, a money laundering transaction might involve an abnormal account and the broader community it interacts with. To address this, we present UniGAD, the first unified framework for detecting anomalies at node, edge, and graph levels jointly. Specifically, we develop the Maximum Rayleigh Quotient Subgraph Sampler (MRQSampler) that unifies multi-level formats by transferring objects at each level into graph-level tasks on subgraphs. We theoretically prove that MRQSampler maximizes the accumulated spectral energy of subgraphs (i.e., the Rayleigh quotient) to preserve the most significant anomaly information. To further unify multi-level training, we introduce a novel GraphStitch Network to integrate information across different levels, adjust the amount of sharing required at each level, and harmonize conflicting training goals. Comprehensive experiments show that UniGAD outperforms both existing GAD methods specialized for a single task and graph prompt-based approaches for multiple tasks, while also providing robust zero-shot task transferability. All codes can be found at \url{https://github.com/lllyyq1121/UniGAD}.
\end{abstract}

\section{Introduction}

Graph Anomaly Detection (GAD) involves identifying a minority of uncommon graph objects that significantly deviate from the majority within graph-structured data \cite{textbook, akoglu2015graph}. These anomalies can manifest as abnormal nodes, unusual relationships, irregular substructures within the graph, or entire graphs that deviate significantly from others. GAD has many practical applications in various contexts, including the identification of bots and fake news on social media \cite{anand2017anomaly,aimeur2023fake,bondielli2019survey, lin2024maximum}, detection of sensor faults and internet invasions in IoT networks \cite{da2019internet, gaddam2020detecting}, and prevention of fraudsters and money laundering activities in transaction networks \cite{hilal2022financial,weber2019anti}. The mainstream GAD models originate from the Graph Neural Networks (GNNs), which have recently gained popularity for mining graph data \cite{GAT, GCN, GIN, GraphSAGE}.  To address the specific challenges of graph anomalies such as label imbalance \cite{PCGNN, DAGAD}, relation camouflage \cite{CareGNN, liu2020alleviating}, and feature heterophily \cite{BWGNN, GHRN}, numerous adaptations of standard GNNs have been proposed \cite{GAD_timeseries, zheng2019addgraph, GRADATE, MAG, mulgad, dominant, wang2019semi, geniepath, zhang2021fraudre}. 

However, existing GAD approaches typically focus on a single type of graph object, such as node-level or graph-level anomaly detection, often overlooking the inherent correlations between different types of objects in graph-structured data. For example, a money laundering transaction might involve both an abnormal account and the broader community it interacts with, while the specific cancer of a cell is determined by particular proteins or protein complexes within the cell. Although some unsupervised methods fuse information from nodes, edges, and subgraphs through reconstruction \cite{li2017radar, dominant, roy2024gad} or contrastive pre-training \cite{xu2022contrastive, GRADATE, liu2021anomaly}, they are still limited to single-level label supervision or prediction. There is a need for a unified approach that considers these correlations information across different levels and performs multi-level anomaly detection. 

To design a unified model for addressing multi-level GAD, we identify two key challenges:\\
\textbf{1. \textit{How to unify multi-level formats?}} Addressing node-level, edge-level, and graph-level tasks uniformly is challenging due to their inherent differences. Some recent works provide insights into unifying these tasks through the use of large language models (LLMs) or prompt tuning. While some methods leverage the generalization capability of LLMs \cite{liu2023one, wang2024instructgraph, li2023survey} on text-attributed graphs, such semantic information is often unavailable in anomaly detection scenarios due to privacy concerns. On the other hand, graph prompt learning methods \cite{sun2023all, liu2023graphprompt, yu2023multigprompt} design induced $k$-hop graphs to transform node or edge levels into graph-level tasks. Nevertheless, their sampling strategies are not specifically tailored to anomaly data, resulting in inappropriate node selections that `erase' critical anomaly information. This oversight can severely impact the effectiveness of anomaly detection.\\
\textbf{2. \textit{How to unify multi-level training?}} Training a single model for multi-level tasks involves various influencing factors, such as transferring information between different levels and achieving a balanced training of these level tasks. There is limited research on multi-task learning in the graph learning domain. Efforts like ParetoGNN \cite{ju2022multi} employ multiple self-supervised learning objectives (e.g., similarity, mutual information) as independent tasks, but these are insufficient for managing multi-level supervision. A comprehensive approach is needed to effectively integrate and balance the training of different level tasks in multi-level GAD.

In this paper, we propose UniGAD, a unified GAD model that leverages the transferability of information across node-level, edge-level, and graph-level tasks. To address the first challenge, we develop a novel subgraph sampler, \textbf{MRQSampler}, that maximizes accumulated spectral energy (i.e., the Rayleigh quotient) in the sampled subgraph with theoretical guarantee, ensuring that the sampled subgraphs contain the most critical anomaly information from nodes and edges. For the second challenge, we introduce the \textbf{GraphStitch} Network, which unifies multi-level training by integrating separate but identical networks for nodes, edges, and graphs into a unified multi-level model. This is achieved using a novel GraphStitch Unit that facilitates information sharing across different levels while maintaining the effectiveness of individual tasks. We perform comprehensive experiments on 14 GAD datasets and compare 17 state-of-the-art methods covering both node-level and graph-level GAD techniques, as well as prompt-based general multi-task graph learning methods. Results show that UniGAD achieves superior performance and offers robust zero-shot transferability across different tasks.

\begin{figure}[t]
  \centering
  \includegraphics[width=1\linewidth]{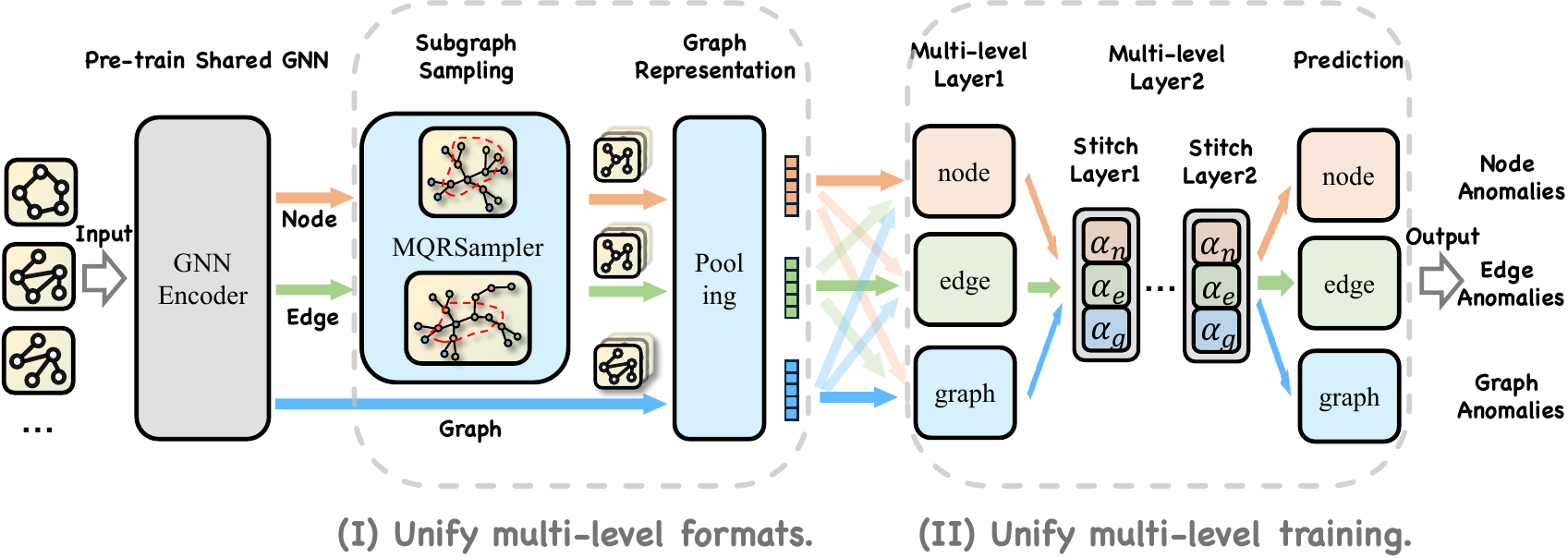}
  \vspace{-5mm}
  \caption{The overall framework of UniGAD.}
  \vspace{-5mm}
  \label{fig:overview}
\end{figure}

\section{Related Work and Preliminaries}

\noindent\textbf{Graph Anomaly Detection.} Leveraging deep learning techniques in GAD has led to significant advancements and a wide range of applications \cite{anand2017anomaly,gaddam2020detecting,aimeur2023fake, bondielli2019survey,hilal2022financial, zhao2023effective}, thoroughly reviewed in a comprehensive survey \cite{ma2021comprehensive}. Node-level anomaly detection, the most prevalent scenario in GAD, has witnessed numerous adaptations and improvements in graph neural networks (GNNs) aimed at enhancing performance from either a spatial \cite{PCGNN, liu2020alleviating, GAS} or spectral \cite{li2019specae, BWGNN, GHRN} perspective. Despite these advancements, recent benchmarks such as BOND \cite{BOND} for unsupervised settings and GADBench \cite{tang2024gadbench} for supervised settings reveal that no single model excels across all datasets, highlighting the need for model selection tailored to specific datasets and task characteristics. For graph-level anomaly detection, various methodologies have been proposed, including transformation learning \cite{OCGIN}, knowledge distillation \cite{GLocalKD}, and evolutionary mapping \cite{GmapAD}. SIGNET \cite{liu2024towards} employs information bottleneck to generate informative subgraphs for explaining graph-level anomalies, while Rayleigh Quotient GNN \cite{dong2023rayleigh} explores the spectral properties of anomalous graphs. Although both node-level and graph-level anomaly detection are rapidly evolving fields, to the best of our knowledge, there is no existing model that supports the joint detection of both node-level and graph-level anomalies.

\noindent\textbf{Multi-task Learning on Graphs.} Multi-task learning involves training a model to handle multiple tasks simultaneously, utilizing shared representations and relationships within the graph to enhance performance across all tasks. Recently, techniques such as graph prompt-based approaches and large language model (LLM)-based approaches have shown promise in this area. Prompt frameworks \cite{zi2024prog} like GraphPrompt \cite{liu2023graphprompt}, All-in-One \cite{sun2023all}, PRODIGY \cite{huang2024prodigy}, MultiGPrompt \cite{yu2023multigprompt}, and SGL-PT \cite{zhu2023sgl} are designed to address a wide array of graph tasks. These approaches transform tasks at other levels into graph-level tasks by leveraging induced graphs. The All-in-One framework enhances connectivity by adding links between the prompt graph and the original graph, whereas GraphPrompt inserts the prompt token into graph nodes through element-wise multiplication. On the other hand, LLM-based frameworks \cite{liu2023one, wang2024instructgraph, li2023survey,chen2024graphwiz,tang2024grapharena} utilize the power of LLMs to learn from different levels, but they require graphs with text attributes or descriptions, which are not applicable in most anomaly detection scenarios. Additionally, some multi-task GNN efforts \cite{ju2022multi} focus on multiple self-supervised specific objectives (such as similarity and mutual information) as independent tasks, which are not suitable for unifying GAD with multi-level label supervision and prediction.

\noindent\textbf{Notation.} 
Let $\mathcal{G} = \{\mathcal{V}, \mathcal{E}, \boldsymbol{X}\}$ denote a connected undirected graph, where $\mathcal{V} = \{v_1, v_2,...,v_N\}$ is the set of $N$ nodes, $\mathcal{E}=\{e_{ij}\}$ is the set of edges, and $\boldsymbol{X} \in \mathbb{R}^{n \times F}$ is node features. Let $\boldsymbol{A}$ be the corresponding adjacency matrix, $\boldsymbol{D}$ be the degree matrix with $\boldsymbol{D}_{i i}=\sum_j \boldsymbol{A}_{i j}$. Laplacian matrix $\boldsymbol{L}$ is then defined as $\boldsymbol{D}-\boldsymbol{A}$ (regular) or as $\boldsymbol{I}-\boldsymbol{D}^{-\frac{1}{2}} \boldsymbol{A} \boldsymbol{D}^{-\frac{1}{2}}$ (normalized), where $\boldsymbol{I}$ is an identity matrix. The Laplacian matrix is a symmetric matrix and can be eigen-decomposed as $\boldsymbol{L}=\boldsymbol{U} \boldsymbol{\Lambda} \boldsymbol{U}^T$, where the diagonal matrix $\boldsymbol{\Lambda}$ consists of real eigenvalues (graph spectrum). Besides, we define the subgraph as $\mathcal{G}_i$ centered on node $v_i$ and our sampled subgraph for node $v_i$ as $\mathcal{S}_i$.

\noindent\textbf{Problem Formulation.} The multi-level graph anomaly detection problem introduces a more universal challenge compared to traditional single-level approaches, described as follows:
\begin{definition}[Multi-level GAD]
Given a training set $\mathcal{T}{r}(\mathcal{N}, \mathcal{E}, \mathcal{G})$ containing nodes, edges, and graphs with arbitrary labels at any of these levels, the goal is to train a unified model to predict anomalies in a test set $\mathcal{T}{e}(\mathcal{N}, \mathcal{E}, \mathcal{G})$, which also contains arbitrary labels at any of these levels.
\end{definition}
Note that our approach does not require the presence of labels at all three levels simultaneously. It is feasible to have labels at one or more levels. Our proposed model aims to leverage the transferability of information across different levels to enhance its predictive capability.

\section{Methodology}
\label{method}

This section details the proposed model UniGAD for multi-level GAD, comprising a GNN encoder, MRQSampler, and GraphStitch Network, as shown in Fig. \ref{fig:overview}. Firstly, a shared pre-trained unsupervised GNN encoder is utilized to learn a more generalized node representation. To unify multi-level formats, the MRQSampler employs spectral sampling to extract subgraphs that contain the highest amount of anomalous information from nodes and edges, thus converting tasks at all three levels into graph-level tasks (Sec. \ref{sample}). To unify multi-level training, the GraphStitch Network integrates information from different levels, adjusts the amount of sharing required at each level, and harmonizes conflicting training goals.  (Sec. \ref{unimodel}).

\subsection{Spectral Subgraph Sampler for Unifying Multi-level Formats}
\label{sample}
In this subsection, we present the Maximum Rayleigh Quotient Subgraph Sampler (MRQSampler), the core module of our unified framework. By sampling subgraphs of nodes or edges, we transform node-level and edge-level tasks into graph-level tasks. Our sampler optimizes these subgraphs to maximize the Rayleigh quotient, ensuring that the sampled subgraphs retain a higher concentration of anomaly information.

\subsubsection{Analysis of the Subgraph Sampling}
\label{samp_motivation}

\titleformat*{\paragraph}{\itshape\bfseries}

\noindent\textbf{What is a suitable subgraph for GAD?} Existing methods on selecting subgraphs for target nodes or edges often use straightforward approaches like $r$-ego or $k$-hop subgraphs \cite{sun2023all}. However, the size of the subgraph is critical for classification outcomes. If the subgraph is too large, it includes too many irrelevant nodes, while if it is too small, it may not align effectively with graph-level tasks.

To measure anomaly information in a subgraph, recent studies \cite{BWGNN, dong2023rayleigh} have identified a `right-shift' phenomenon in the spectral energy distribution, moving from low to higher frequencies. This accumulated spectral energy can be quantified by the Rayleigh quotient \cite{horn2012matrix}:
\begin{equation}
\label{RQ}
RQ(\boldsymbol{x},\boldsymbol{L}) = \frac{\boldsymbol{x}^T \boldsymbol{L} \boldsymbol{x}}{\boldsymbol{x}^T \boldsymbol{x}} =  \frac{\sum_{(i, j) \in \mathcal{E}} A_{ij}(x_j-x_i)^2}{\sum_{i \in \mathcal{V}} x_i^2}.
\end{equation}
The following lemma \cite{BWGNN} illustrates the relationship between the Rayleigh quotient $RQ(\boldsymbol{x},\boldsymbol{L})$ and anomaly information:
\begin{lemma}[Tang, 2022]
    Rayleigh quotient $RQ(\boldsymbol{x},\boldsymbol{L})$, i.e. the accumulated spectral energy of the graph signal, is monotonically increasing with the anomaly degree.
\end{lemma}
Thus, for any node $v_i$, our sampling objective is to identify the induced subgraph with the highest Rayleigh quotient containing the most anomaly information.

\noindent\textbf{Where to Sample Subgraph From?} To preserve the properties of target nodes, it is essential to sample subgraphs centered around these nodes, capturing key surrounding nodes. The most intuitive methods are $r$-ego graphs or $k$-hop graphs. However, considering the message-passing mechanisms of most GNNs \cite{GCN, wu2019simplifying, GraphSAGE}, a classical work \cite{GIN} provides valuable insight:
\begin{lemma}[Xu, 2018]
    A GNN recursively updates each node's feature vector through its rooted subtree structures to capture the network structure and features of surrounding nodes. 
\end{lemma}
\begin{figure}[t]
  \centering
  \includegraphics[width=0.8\linewidth]{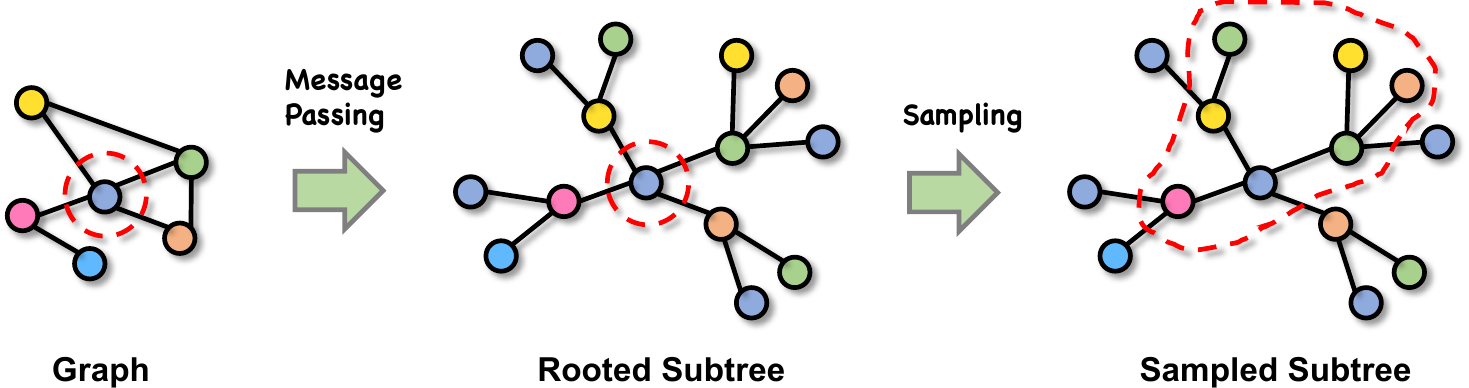}
  \caption{Message passing in GNNs and rooted subtree sampling.}
  \label{fig:subtree}
\end{figure}
As shown in Fig. \ref{fig:subtree}, the message-passing process of GNNs suggests that a rooted subtree centered on the target node is more consistent with the GNN's architecture. Therefore, we sample subgraphs from these rooted subtree structures. The remaining question is: \textbf{\textit{How to implement subgraph sampling based on the above?}} To address this, we introduce a novel MRQSampler in the next subsection.

\titleformat*{\paragraph}{\normalfont\normalsize\bfseries}

\subsubsection{Maximum Rayleigh Quotient Subgraph Sampler (MRQSampler)}
\label{MRQSampler}

Building on the motivation in Section \ref{samp_motivation}, our approach involves sampling subgraphs for each node starting from the rooted subtree with the node as its root. The target node is always included. We then select the subtree with the maximum Rayleigh quotient from all possible subtrees as the representative subgraph for that node to ensure it contain the maximum anomaly information. We formulate this as the following optimization problem:
\begin{equation}
\begin{aligned}
    \mathcal{S}^\star  = \mathop{\arg\max}_{{\mathcal{S}} \subseteq \mathcal{G}} \quad &  \frac{\sum_{(p, q) \in \mathcal{E}_{\mathcal{S}}} (x_p-x_q)^2}{\sum_{p \in \mathcal{S}} x_p^2}, \\
    \text{s.t.} \quad \quad & v \in \mathcal{S},
     \\ & \forall v_p \in \mathcal{S}, \ (v, v_p ) \  \text{is accessible.}
\end{aligned}
\end{equation}
where $\mathcal{G}$ represents $k$-depth rooted subtree from $v$, and $\mathcal{S}$ is a possible subgraph from $\mathcal{G}$. The first constraint ensures the target node is included, and the second constraint ensures message passability. Generally, similar selecting subgraphs in this manner is considered an NP-Hard problem \cite{yangraye}. However, leveraging the properties of trees, we propose an algorithm to solve \textbf{the optimal solution}.

We first determine the conditions that increase a subgraph's Rayleigh quotient when adding a node, presented in the following theorem:
\begin{theorem}
\label{theorem1}
For a graph $\mathcal{G}$, let one of its subgraphs be $\mathcal{S}$, and let its Rayleigh quotient be $RQ(\mathcal{S})$. If a new node $v_{new} \in \mathcal{G}-\mathcal{S}$ is added to $\mathcal{S}$, the Rayleigh quotient $RQ(\mathcal{S})$ will increase if and only if:
\begin{equation}
\label{proposition1}
\Delta(v_{new}) = \frac{\sum_{ v_r \in \mathcal{S}} (x_{new}-x_r)^2}{x_{new}^2} > RQ(\mathcal{S}).
\end{equation}
\end{theorem}
The proof of Theorem \ref{theorem1} can be found in Appendix \ref{proof1}. We can extend this theorem from a single new node $v_{new}$ to a new node set $\mathcal{V}_{new}$, leading to the following corollary:
\begin{corollary}
\label{corollary11}
For a graph $\mathcal{G}$, let one of its subgraphs be $\mathcal{S}$, and let its Rayleigh quotient be $RQ(\mathcal{S})$. If a new nodeset $\mathcal{V}{new} \subset \mathcal{G} -\mathcal{S}$ is added to $\mathcal{S}$, the Rayleigh quotient $RQ(\mathcal{S})$ will increase if and only if:

\begin{equation}
\label{corollary1}
\! \! \!\!\!\! \Delta(\mathcal{V}_{new}) = \frac{\sum_{ (i,r)\in \mathcal{E}_{  \mathcal{S} + \mathcal{V}_{new} } } (x_{new_{i}}-x_r)^2+\sum_{(i, j) \in \mathcal{E}_{\mathcal{V}_{new}}} (x_{new_i}-x_{new_j})^2}{\sum_{v_{new} \in \mathcal{V}_{new}} x_{new}^2} > RQ(\mathcal{S}).
\end{equation}

\end{corollary}
The proof details are also in Appendix \ref{proof1c}. While the above analysis can indeed increase the Rayleigh quotient of the sampled subgraph, the sampling order may cause the results to fall into a local optimum, which may not guarantee a globally optimal solution. To identify the nodes that must be sampled in the optimal subgraph, we present the following theorem:
\begin{theorem}
\label{corollary2}
    For a graph $\mathcal{G}$, let one of its subgraph be $\mathcal{S}$, the $\mathcal{S}^*$ be its final optimal subgraph, and $\mathcal{S} 
    \subset \mathcal{S}^*$. For a new \textbf{connected} nodeset $\tilde{\mathcal{V}}_{new} \cap S = \emptyset $, it is contained in $\mathcal{S}^*$ when it satisfies:
\begin{equation}
\label{eq:new5}
\Delta_{max}(\tilde{\mathcal{V}}_{new}) = \max_{\tilde{\mathcal{V}}_{new} \subseteq \mathcal{G}-\mathcal{S}}\Delta(\tilde{\mathcal{V}}_{new}), \ \text{and} \
\Delta_{max}(\tilde{\mathcal{V}}_{new}) > RQ(\mathcal{S}).
\end{equation}
\end{theorem}
We refer readers to Appendix \ref{proof2} for the rigorous proof. Through the above analysis, we derive the conditions of the nodeset contained in the optimal subtree (Theorem \ref{corollary2}). When $ \tilde{\mathcal{V}}_{new} $ satisfies Eq. (\ref{eq:new5}), it always increases the Rayleigh quotient based on the current subgraph, ensuring that $ \mathcal{V}_{new} $ is contained in the optimal solution. Thus, we decouple the problem of finding the subgraph with the maximum Rayleigh quotient into a process of finding the maximum $\Delta_{max}(\mathcal{V}_{new})$ each time, until adding any node/node set fails to increase the $RQ(\mathcal{S})$. 
Following this, we design a dynamic programming (DP) algorithm to ensure the optimal subset satisfies these conditions.
\begin{figure}[t]
\centering
\includegraphics[width=0.99\linewidth]{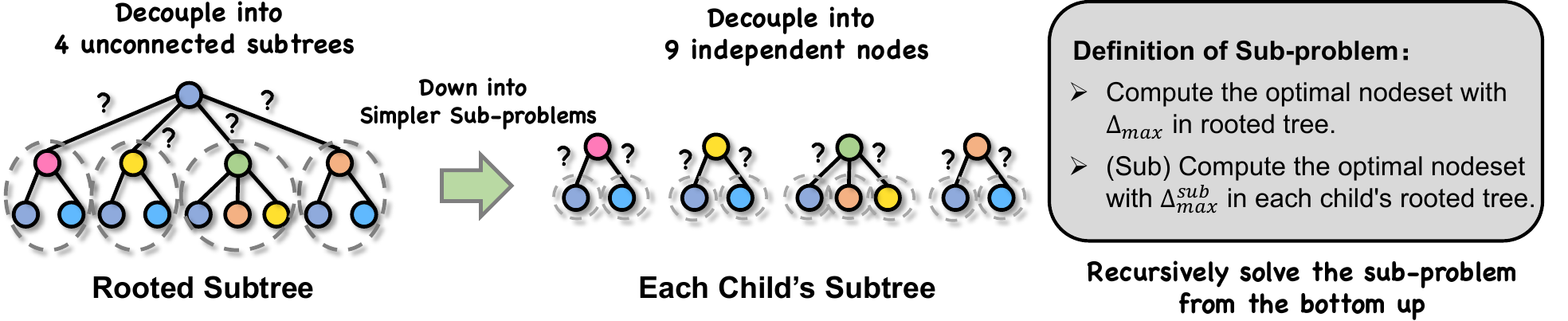}
\caption{MRQSampler: (i) Derive the condition (Theorem \ref{corollary2}) satisfied with the optimal subtree. (ii) Decompose the problem into simpler sub-problems by recursing through the tree depth to solve the optimal subtree with the dynamic programming (DP) algorithm.}\vspace{-3mm}
\label{fig:sampler}
\end{figure}

\noindent\textbf{MRQSampler Algorithm.}
We introduce the Maximum Rayleigh Quotient Subgraph Sampler (MRQSampler), which uses dynamic programming (DP) to find the optimal solution. We break down the computation for the central node into sub-problems, storing the results of sub-problems to avoid redundant computations in future calculations. For a rooted subtree with the target node (edge) as the root, its children are unconnected to each other. In Fig. \ref{fig:sampler}, we consider a 2-depth subtree and summarize our algorithm as follows:
\begin{itemize}[nosep,left=1em]
\item \textbf{Stage 1:} We recursively compute and store the maximum $\Delta(\tilde{\mathcal{V}}_{new})$ for each subtree, which can be down into simpler sub-problems similar to the previous one and calculates each layer in the tree recursively from the bottom up. 
\item \textbf{Stage 2:} Based on Theorem \ref{corollary2}, we iteratively select the descendant with the maximum $\Delta_{max}(\tilde{\mathcal{V}}_{new})$ (within its own rooted subtree) of the target node and the currently selected nodeset, until the conditions of Theorem \ref{corollary2} are no longer satisfied, i.e., when the Rayleigh quotient of the sampled subgraph no longer increases.
\end{itemize}
For efficiency, this approach can obtain the subgraph with the maximum Rayleigh quotient of the target node/edge's rooted subtree while reducing the algorithmic complexity to $O(N \log N)$. It can be further accelerated in parallel since the computation for different nodes is independent. Additionally, the sampling process only needs to be computed once in training and inference processes, minimally impacting model efficiency. For the detailed pseudocode of the algorithm, please refer to Appendix \ref{pseudocode}. Note that we use mean pooling for entire graphs, but for subgraphs, we use weighted pooling to highlight central nodes/edges, with an exponential decay based on the number of hops to the central nodes/edges. This method transforms node-level and edge-level tasks into graph-level tasks, ensuring that the most anomaly information is retained in the sampled subgraphs.

\subsection{GraphStitch Network for Unifying Multi-level Training}
\label{unimodel}
After obtaining graph representations, training them together through a fully connected layer can negatively impact individual levels due to the inherent differences across different-level anomalies. This can result in mediocre performance at all levels. A key challenge, therefore, is to facilitate information transfer between multi-levels without compromising single-level effectiveness. Inspired by work in the computer vision field \cite{misra2016cross}, we introduce the novel GraphStitch Network to jointly consider multi-level representations.

Specifically, we train separate but identical networks for each level and use the GraphStitch unit to combine these networks into a multi-level network, managing the degree of sharing required at different levels. This approach aims to maintain single-level effectiveness while enhancing multi-level information transfer. The network structure is illustrated in Fig. \ref{fig:stitch}.

\begin{figure}[t]
  \centering
  \includegraphics[width=0.8\linewidth]{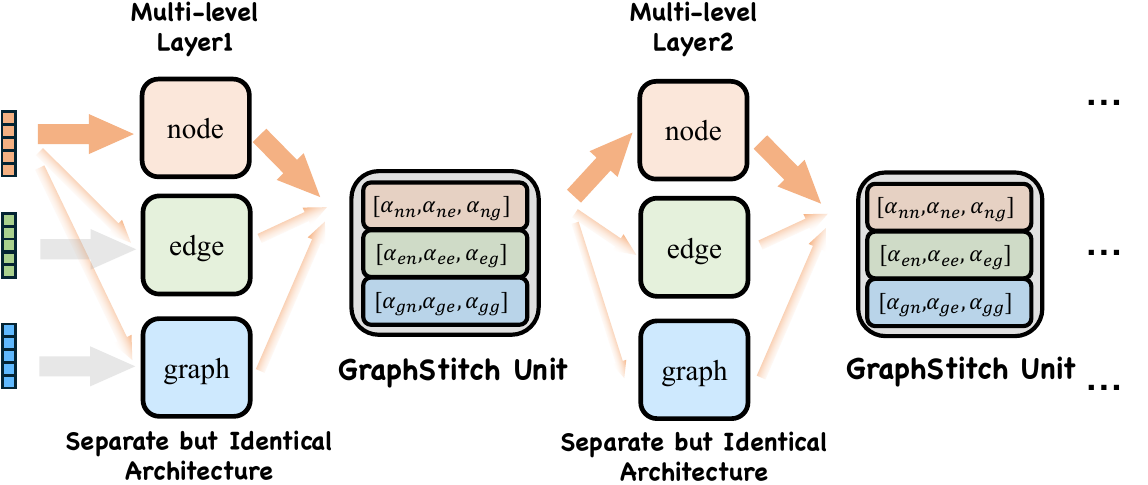} \vspace{-2mm}
  \caption{GraphStitch network structure in UniGAD. Node level is highlighted.} \vspace{-4mm}
  \label{fig:stitch}
\end{figure}

To elaborate, we denote $\mathbf{e}_N$, $\mathbf{e}_E$, and $\mathbf{e}_G$ as the embeddings for nodes, edges, and graphs, respectively. The node embedding $\mathbf{e}_N = (\mathbf{e}_{nn}, \mathbf{e}_{ne}, \mathbf{e}_{ng})^\top$ consists of outputs from three separate but identically structured networks specialized for nodes, edges, and graphs. Similarly, the edge and graph embeddings are represented as $\mathbf{e}_E = (\mathbf{e}_{en}, \mathbf{e}_{ee}, \mathbf{e}_{eg})^\top$ and $\mathbf{e}_G = (\mathbf{e}_{gn}, \mathbf{e}_{ge}, \mathbf{e}_{gg})^\top$.

We define a GraphStitch operation as follows:
\begin{equation}
\left(\tilde{\mathbf{e}}_N, \tilde{\mathbf{e}}_E, \tilde{\mathbf{e}}_G\right)=diag\left[
\left(\begin{array}{lll}
\alpha_{n n} & \alpha_{n e} & \alpha_{n g} \\
\alpha_{e n} & \alpha_{e e} & \alpha_{e g} \\
\alpha_{g n} & \alpha_{g e} & \alpha_{g g}
\end{array}\right)\left(\mathbf{e}_N, \mathbf{e}_E, \mathbf{e}_G\right)\right].
\end{equation}
The sharing of representations is achieved by learning a linear combination of the outputs from the three networks. This linear combination is parameterized using learnable $\mathbf{\alpha}$. In particular, when training data lacks a certain level, the influence of that level on other levels is defined as zero during training but still retains the influence of other levels on this level. In this way, it allows the labels for training and testing to be arbitrary. Besides, if all the cross terms ($\alpha_{ne}$, $\alpha_{ng}$, $\alpha_{en}$, $\alpha_{eg}$, $\alpha_{gn}$, $\alpha_{ge}$) are equal to 0 means that training the three networks jointly is equivalent to training them independently. Finally, the embeddings for nodes, edges, and graphs are fed into three independent multi-layer perceptrons (MLPs) to compute the abnormal probabilities $p_i^\mathcal{N}$, $p_i^\mathcal{E}$, and $p_i^\mathcal{G}$, respectively.

In addition to the GraphStitch structure, UniGAD optimizes the loss functions for multi-level tasks. Specifically, the gradients of each level task's loss may conflict in direction or magnitude, potentially causing negative effects and resulting in worse performance compared to learning single-level tasks individually. Therefore, UniGAD uses a multi-level weighted cross-entropy loss for training:
\begin{equation}
\mathcal{L}=\!\!\!\sum_{\{\mathcal{N},\mathcal{E},\mathcal{G}\}} \!\!\!\! \sum_i \beta^{\{\mathcal{N},\mathcal{E},\mathcal{G}\}} \left[\gamma y_i^{\{\mathcal{N},\mathcal{E},\mathcal{G}\}} \log \left(p_i^{\{\mathcal{N},\mathcal{E},\mathcal{G}\}}\right)+\left(1-y_i^{\{\mathcal{N},\mathcal{E},\mathcal{G}\}}\right) \log \left(1-p_i^{\{\mathcal{N},\mathcal{E},\mathcal{G}\}}\right)\right].
\end{equation}
where $\gamma$ is the ratio of anomaly labels ($y_i=1$) to normal labels ($y_i=0$), and $\beta^{\{\mathcal{N},\mathcal{E},\mathcal{G}\}}$ are adaptive weights for different tasks. We adopt a `Gradient Surgery' approach \cite{yu2020gradient} to adjust the $\beta^{\{\mathcal{N},\mathcal{E},\mathcal{G}\}}$, altering the gradients by projecting each onto the normal plane of the others. This prevents the interfering components of the gradients from affecting the network and minimizes interference among different-level GAD tasks. In this way, UniGAD ensures that each level remains relatively independent while facilitating cross-passing of relevant information between multi-level tasks.

\section{Experiments}
\label{exp}

In this section, we conduct experiments to evaluate our UniGAD with node-level, edge-level, and graph-level tasks by answering the following questions: \textbf{Q1:} How effective is UniGAD in unifying in multi-level anomaly detection? \textbf{Q2:} Can UniGAD transfer information across different levels in zero-shot learning? \textbf{Q3:} What are the contributions of the modular design in the UniGAD model? \textbf{Q4:} How do the time and space efficiencies of UniGAD compare to those of other methods?

\subsection{Experimental Setup}

\noindent\textbf{Datasets.} 
We consider a total of 14 datasets, including both single-graph datasets and multi-graph datasets. 7 single-graph datasets are used to evaluate the performance of unifying node-level and edge-level tasks: Reddit, Weibo, Amazon, Yelp, Tolokers, and Questions, T-finance from the work \cite{tang2024gadbench}, which contain node-level anomaly labels. For edge anomaly labels, we generated them according to a specific anomaly probability following the formula $ P_{\text{anom}}^{(i,j)} = avg (P_{\text{anom}}^i, P_{\text{anom}}^j)$. And 7 multi-graph datasets are used to validate the performance of unifying node-level and graph-level tasks, including BM-MN, BM-MS, BM-MT, MUTAG, MNIST0, MNIST1, and T-Group. The first six datasets are from \cite{liu2024towards}, containing both node anomaly labels and graph anomaly labels. Moreover, we release a real-world large-scale social group dataset T-Group, combining the data (graph anomaly labels) in \cite{li2019semi}. For its node anomaly labels, we assume that if a node appears in 3 malicious social groups, we consider it a malicious node. Statistical data for these datasets can be found in Table \ref{tab:datasets-stat}, including the percentage of training data, the number of graphs, edges, nodes, feature dimensions, and the proportions of abnormal nodes, edges, and graphs (Nodes$_{ab}$, Edges$_{ab}$, and Graphs$_{ab}$).

\begin{table}[h]
    \caption{Detailed statistics of the datasets used in our experiments.}
    \label{tab:datasets-stat}
    \setlength{\tabcolsep}{2pt}
    \centering
    \begin{adjustbox}{width=0.8\textwidth}
        \begin{tabular}{@{}ccccccccc@{}}
            \toprule
            \textbf{Dataset} & \textbf{Train\%}    & \textbf{\# Graphs} & \textbf{\# Edges} & \textbf{\# Nodes} & \textbf{\# Dims} & \textbf{Nodes$_{ab}$} & \textbf{Edges$_{ab}$} & \textbf{Graphs$_{ab}$}  \\ \midrule
            Reddit           & 40\%                    & 1         & 168,016                & 10,984                 & 64                     & 3.33\%                         & 2.72\%                         &     /       \\
            Weibo            & 40\%                    & 1         & 416,368                & 8,405                  & 400                    & 10.33\%                        & 5.71\%                         &     /       \\
            Amazon           & 70\%                    & 1         & 8,847,096               & 11,944                 & 25                     & 6.87\%                         & 2.49\%                         &     /       \\
            Yelp             & 70\%                    & 1         & 7,739,912               & 45,954                 & 32                     & 14.53\%                        & 13.89\%                        &     /       \\
            Tolokers         & 50\%                    & 1         & 530,758                 & 11,758                 & 10                     & 21.82\%                        & 33.44\%                        &     /       \\
            Questions        & 50\%                    & 1         & 202,461                & 48,921                 & 301                    & 2.98\%                         & 7.50\%                         &     /       \\
            T-Finance        & 40\%                    & 1         & 21,222,543                & 39,357                 & 10                    & 4.58\%                         & 2.77\%                         &     /       \\ 
            \midrule
            BM-MN      & 40\%                    & 700                  & 40,032                 & 12,911                 & 1                      & 48.91\%                             &              /                 & 14.29\%    \\
            BM-MS      & 40\%                    & 700                  & 30,238                 & 9,829                  & 1                      & 31.99\%                             &              /                 & 14.29\%    \\
            BM-MT      & 40\%                    & 700                  & 32,042                 & 10,147                 & 1                      & 34.49\%                             &              /                 & 14.29\%    \\ 
            MUTAG           & 40\%                    & 2,951                 & 179,732                & 88,926                 & 14                     & 4.81\%                         &              /                 & 34.40\%    \\
            MNIST0           & 10\%                    & 70,000                & 41,334,380              & 4,939,668               & 5                      & 35.46\%                        &              /                 & 9.86\%     \\
            MNIST1           & 10\%                    & 70,000                & 41,334,380              & 4,939,668               & 5                      & 35.46\%                        &              /                 & 11.25\%    \\
            T-Group          & 40\%                    & 37,402                & 93,367,082              & 11,015,616              & 10                     & 0.64\%                         &              /                 & 4.26\%     \\ \bottomrule
            \end{tabular}
\end{adjustbox}
\end{table}

\noindent\textbf{Baselines.} 
To comprehensively compare with traditional single-level tasks, we consider nine representative node-level methods: GCN \cite{GCN}, GIN \cite{GIN}, GraphSAGE \cite{GraphSAGE}, SGC \cite{wu2019simplifying}, GAT \cite{GAT}, BernNet \cite{BernNet}, PNA \cite{PNA}, AMNet \cite{AMNet}, and BWGNN \cite{BWGNN}. Given the limited work on edge anomalies, we adapt a concatenated strategy \cite{zhang2022graph} from link prediction, resulting in nine corresponding edge-level methods: GCNE, GINE, GSAGEE, SGCE, GATE, BernE, PNAE, AME, and BWE. For graph-level anomaly detection, we consider six state-of-the-art methods: OCGIN \cite{OCGIN}, OCGTL \cite{OCGTL}, GLocalKD \cite{GLocalKD}, iGAD \cite{iGAD}, GmapAD \cite{GmapAD}, and RQGNN \cite{dong2023rayleigh}. Additionally, to compare multi-task models, we include two recent multi-task graph prompt methods: GraphPrompt \cite{liu2023graphprompt} and All-in-One \cite{sun2023all}. While these methods were not originally proposed for joint multi-task training, we adapt their ideas and develop multi-task versions for our comparison, GraphPrompt-U and All-in-One-U, whose modifications were limited to the data preprocessing component to accommodate the simultaneous handling of multiple object types (node/edge or node/graph) within induced graphs. 

\noindent\textbf{Implementations.} We evaluate three metrics: AUROC (Sec. \ref{exp}),  Macro F1-score and AUPRC (Appendix \ref{addsesult}). For each result, we conduct 5 runs and report the mean results. In UniGAD, we choose two backbone GNN encoders: GCN \cite{GCN} and BWGNN \cite{BWGNN}. We use a shared graph pre-training method, GraphMAE \cite{GraphMAE}, to obtain a more generalized node representation. For multi-dimensional feature vectors, we normalize all feature dimensions and then take the norm (1-norm in our case) to obtain a composite feature for each node, allowing us to identify the most anomalous nodes in MRQSampler based on this comprehensive feature.
To avoid data leakage, for single-graph datasets, edges between the training set and the testing set are not considered; for multi-graph datasets, the training set and the testing set consist of different graphs and their nodes. More details on the implementation can be found in the Appendix \ref{expdetails}.

\subsection{Multi-Level Performance Comparison (RQ1)}
To compare the performance of multi-level anomaly detection, we conduct experiments under two settings. For the single-graph datasets, we compare the performance of unified training on node-level and edge-level data. For the multi-graph datasets, we compare the performance of unified training on node-level and graph-level data.

\begin{table}[t]
    \caption{Comparison of unified performance (AUROC) at both node and edge levels with different single-level methods, multi-task methods, and our proposed method.}
    \label{NEfull}
    \centering
    \resizebox{\textwidth}{!}{
    \begin{tabular}{c|c|cc|cc|cc|cc|cc|cc|cc}
    \toprule
    \multirow{2}{*}{ } & \bf{Dataset} & \multicolumn{2}{c|}{\bf{Reddit}} & \multicolumn{2}{c|}{\bf{Weibo}} & \multicolumn{2}{c|}{\bf{Amazon}} & \multicolumn{2}{c|}{\bf{Yelp}} & \multicolumn{2}{c|}{\bf{Tolokers}} & \multicolumn{2}{c|}{\bf{Questions}}  & \multicolumn{2}{c}{\bf{T-Finance}} \\ 
     & \bf{Task-level} & \bf{Node} & \bf{Edge}  & \bf{Node} & \bf{Edge} & \bf{Node} & \bf{Edge}  & \bf{Node} & \bf{Edge}  & \bf{Node} & \bf{Edge}  & \bf{Node} & \bf{Edge} & \bf{Node} & \bf{Edge} \\ \midrule
    \multirow{9}{*}{Node-Level} 
    & GCN & $62.60$ & / & $97.97$ & / & $82.37$ & / & $57.62$ & / & $75.21$ & / & $70.15$ & / &$90.70$ & /\\
     & GIN & $65.59$ & / & $95.64$ & / & $92.17$ & / & $74.46$ & / & $75.15$ & / & $69.13$ & / &$86.43$ & /\\
     & GraphSAGE & $62.25$ & / & $94.45$ & / & $84.53$ & / & $82.12$ & / & $79.74$ & / & $72.47$ & /&$78.16$ & / \\
     & SGC & $52.12$ & / & $97.71$ & / & $80.24$ & / & $53.03$ & / & $69.51$ & / & $70.59$ & /& $74.21$ & / \\
     & GAT & $65.87$ & / & $94.40$ & / & $96.24$ & / & $77.40$ & / & $78.90$ & / & $71.38$ & / &$90.60$ & / \\
     & BernNet & $66.68$ & / & $93.93$ & / & $96.62$ & / & $81.48$ & / & $76.68$ & / & $70.28$ & / &$92.37$ & / \\
     & PNA & $65.28$ & / & $97.43$ & / & $81.41$ & / & $71.81$ & / & $75.82$ & / & $71.78$ & /& $68.17$ & / \\
     & AMNet & $68.31$ & / & $94.17$ & / & $97.31$ & / & $81.42$ & / & $76.67$ & / & $68.63$ & / &$93.58$ & / \\
     & BWGNN & $64.65$ & / & $97.42$ & / & $97.80$ & / & $83.11$ & / & $80.51$ & / & $70.25$ & /& $96.03$ & / \\ \midrule
    \multirow{9}{*}{Edge-level} 
    & GCNE & / & \multicolumn{1}{l|}{$63.10$} & / & \multicolumn{1}{l|}{$99.03$} & / & \multicolumn{1}{l|}{$78.63$} & / & \multicolumn{1}{l|}{$57.80$} & / & \multicolumn{1}{l|}{$73.59$} & / & \multicolumn{1}{l|}{$79.05$} & / & \multicolumn{1}{l}{$87.63$}\\
     & GINE & / & \multicolumn{1}{l|}{$67.36$} & / & \multicolumn{1}{l|}{$98.09$} & / & \multicolumn{1}{l|}{$79.74$} & / & \multicolumn{1}{l|}{$67.58$} & / & \multicolumn{1}{l|}{$69.27$} & / & \multicolumn{1}{l|}{$80.75$} & / & \multicolumn{1}{l}{$79.05$} \\
    & GSAGEE & / & \multicolumn{1}{l|}{$\bf{67.52}$} & / & \multicolumn{1}{l|}{$98.67$} & / & \multicolumn{1}{l|}{$78.92$} & / & \multicolumn{1}{l|}{$73.30$} & / & \multicolumn{1}{l|}{$\bf{76.98}$} & / & \multicolumn{1}{l|}{$\bf{87.51}$}& / & \multicolumn{1}{l}{$77.14$} \\
     & SGCE & / & \multicolumn{1}{l|}{$53.36$} & / & \multicolumn{1}{l|}{$98.55$} & / & \multicolumn{1}{l|}{$76.41$} & / & \multicolumn{1}{l|}{$52.02$} & / & \multicolumn{1}{l|}{$70.59$} & / & \multicolumn{1}{l|}{$74.24$} & / & \multicolumn{1}{l}{$69.01$}\\
     & GATE & / & \multicolumn{1}{l|}{$67.07$} & / & \multicolumn{1}{l|}{$97.92$} & / & \multicolumn{1}{l|}{$90.20$} & / & \multicolumn{1}{l|}{$72.96$} & / & \multicolumn{1}{l|}{$71.92$} & / & \multicolumn{1}{l|}{$81.64$} & / & \multicolumn{1}{l}{$83.09$}\\
     & BernE & / & \multicolumn{1}{l|}{$65.57$} & / & \multicolumn{1}{l|}{$97.87$} & / & \multicolumn{1}{l|}{$89.60$} & / & \multicolumn{1}{l|}{$73.93$} & / & \multicolumn{1}{l|}{$73.39$} & / & \multicolumn{1}{l|}{$84.78$} & / & \multicolumn{1}{l}{$87.80$} \\
     & PNAE & / & \multicolumn{1}{l|}{$64.15$} & / & \multicolumn{1}{l|}{$99.10$} & / & \multicolumn{1}{l|}{$75.71$} & / & \multicolumn{1}{l|}{$67.98$} & / & \multicolumn{1}{l|}{$75.09$} & / & \multicolumn{1}{l|}{$84.05$} & / & \multicolumn{1}{l}{$83.91$}\\
     & AME & / & \multicolumn{1}{l|}{$66.73$} & / & \multicolumn{1}{l|}{$97.08$} & / & \multicolumn{1}{l|}{$89.36$} & / & \multicolumn{1}{l|}{$73.69$} & / & \multicolumn{1}{l|}{$71.99$} & / & \multicolumn{1}{l|}{$84.93$}& / & \multicolumn{1}{l}{$86.19$} \\
     & BWE & / & \multicolumn{1}{l|}{$67.39$} & / & \multicolumn{1}{l|}{$98.93$} & / & \multicolumn{1}{l|}{$91.61$} & / & \multicolumn{1}{l|}{$75.63$} & / & \multicolumn{1}{l|}{$75.66$} & / & \multicolumn{1}{l|}{$85.00$}& / & \multicolumn{1}{l}{$92.27$} \\ 
     \midrule
    \multirow{2}{*}{Multi-task} 
    & GraphPrompt-U & $50.03$  & $49.78$  & $55.29$  & $50.71$  & $50.01$  & $50.96$ & $49.83$ & $49.56$ & $51.24$  & $49.66$  & $55.16$  & $50.01$ & OOT & OOT  \\
    & All-in-One-U  & $51.35$  & $54.10$  & $48.61$  & $52.63$  & $56.11$  & $54.80$ & $49.77$ & $49.13$ & $50.41$  & $49.29$  & $51.49$  & $64.24$ & OOT & OOT  \\
    
     \midrule
    \multirow{2}{*}{\begin{tabular}[c]{@{}c@{}}UniGAD\\ (Ours)\end{tabular}} 
    & UniGAD - GCN & \multicolumn{1}{l}{$\bf{71.65}$} & \multicolumn{1}{l|}{$65.46$} & \multicolumn{1}{l}{$99.02$} & \multicolumn{1}{l|}{$\bf{99.13}$} & \multicolumn{1}{l}{$82.92$} & \multicolumn{1}{l|}{$80.04$} & \multicolumn{1}{l}{$63.22$} & \multicolumn{1}{l|}{$61.74$} & \multicolumn{1}{l}{$77.26$} & \multicolumn{1}{l|}{$72.89$} & \multicolumn{1}{l}{$\bf{73.92}$} & \multicolumn{1}{l|}{$74.72$} & \multicolumn{1}{l}{$95.68$} & \multicolumn{1}{l}{$93.75$}\\
     &  UniGAD - BWG & \multicolumn{1}{l}{$64.42$} & \multicolumn{1}{l|}{$53.60$} & \multicolumn{1}{l}{$\bf{99.07}$} & \multicolumn{1}{l|}{$99.10$} & \multicolumn{1}{l}{$\bf{97.84}$} & \multicolumn{1}{l|}{$\bf{92.18}$} & \multicolumn{1}{l}{$\bf{86.23}$} & \multicolumn{1}{l|}{$\bf{79.05}$} & \multicolumn{1}{l}{$\bf{80.62}$} & \multicolumn{1}{l|}{$74.85$} & \multicolumn{1}{l}{$70.97$} & \multicolumn{1}{l|}{$73.45$} & \multicolumn{1}{l}{$\bf{96.49}$} & \multicolumn{1}{l}{$\bf{94.32}$}\\ 
    
    \bottomrule
    \end{tabular}
     }    
    \end{table}

\begin{table}[t]
\vspace{-3mm}
    \caption{Comparison of unified performance (AUROC) at both node and graph levels with different single-level methods, multi-task methods, and our proposed method.}
    \label{NGfull}
    \centering
    \resizebox{\textwidth}{!}{
    \begin{tabular}{c|c|cc|cc|cc|cc|cc|cc|cc}
    \toprule
    \multirow{2}{*}{ } & \bf{Dataset} & \multicolumn{2}{c|}{\bf{BM-MN}} & \multicolumn{2}{c|}{\bf{BM-MS}} & \multicolumn{2}{c|}{\bf{BM-MT}} & \multicolumn{2}{c|}{\bf{MUTAG}} & \multicolumn{2}{c|}{\bf{MNIST0}} & \multicolumn{2}{c|}{\bf{MNIST1}}  & \multicolumn{2}{c}{\bf{T-Group}} \\ 
     & \bf{Task-level} & \bf{Node} & \bf{Graph}  & \bf{Node} & \bf{Graph} & \bf{Node} & \bf{Graph}  & \bf{Node} & \bf{Graph}  & \bf{Node} & \bf{Graph}  & \bf{Node} & \bf{Graph}  & \bf{Node} & \bf{Graph} \\ \midrule
    \multirow{9}{*}{Node-level} 
    & GCN       & $86.31$ & / & $90.17$ & / & $92.30$ & / & $99.38$ & / & $94.10$  & / & $93.84$  & / & $91.81$ & / \\
    & GIN       & $56.73$ & / & $50.41$ & / & $54.90$ & / & $99.39$ & / & $93.55$  & / & $93.49$  & / & $61.51$ & / \\
    & GraphSAGE & $50.00$ & / & $50.00$ & / & $49.95$ & / & $99.26$ & / & $\bf{99.99}$ & / & $\bf{99.99}$ & / & $64.15$ & / \\
    & SGC       & $50.27$ & / & $50.87$ & / & $49.44$ & / & $89.19$ & / & $86.97$  & / & $86.97$  & / & $82.55$ & / \\
    & GAT       & $58.47$ & / & $62.52$ & / & $65.72$ & / & $99.42$ & / & $99.90$  & / & $\bf{99.99}$  & / & $78.17$ & / \\
    & BernNet   & $60.06$ & / & $65.58$ & / & $59.18$ & / & $98.97$ & / & $\bf{99.99}$ & / & $\bf{99.99}$ & / & $93.85$ & / \\
    & PNA       & $72.96$ & / & $55.19$ & / & $75.61$ & / & $98.76$ & / & $99.80$  & / & $99.87$  & / & $55.66$ & / \\
    & BWGNN     & $93.05$ & / & $87.22$ & / & $88.97$ & / & $99.50$ & / & $\bf{99.99}$ & / & $\bf{99.99}$ & / & $94.81$ & / \\ \midrule
    \multirow{6}{*}{Graph-level} 
    & OCGIN    & / & \multicolumn{1}{l|}{$98.46$} & / & \multicolumn{1}{l|}{$81.97$} & / & \multicolumn{1}{l|}{$58.05$} & / & \multicolumn{1}{l|}{$89.50$} & / & \multicolumn{1}{l|}{$57.24$} & / & \multicolumn{1}{l|}{$86.15$} & / & 64.53 \\
    & OCGTL    & / & \multicolumn{1}{l|}{$98.48$} & / & \multicolumn{1}{l|}{$83.17$} & / & \multicolumn{1}{l|}{$59.99$} & / & \multicolumn{1}{l|}{$92.19$} & / & \multicolumn{1}{l|}{$59.35$} & / & \multicolumn{1}{l|}{$93.45$} & / & 46.77 \\
    & GLocalKD & / & \multicolumn{1}{l|}{$92.36$} & / & \multicolumn{1}{l|}{$77.25$} & / & \multicolumn{1}{l|}{$53.23$} & / & \multicolumn{1}{l|}{$72.77$} & / & \multicolumn{1}{l|}{$66.69$} & / & \multicolumn{1}{l|}{$57.42$} & / & 78.53 \\
    & iGAD     & / & \multicolumn{1}{l|}{$91.68$}     & / & \multicolumn{1}{l|}{$96.68$}     & / & \multicolumn{1}{l|}{$99.14$}     & / & \multicolumn{1}{l|}{$96.28$}     & / & \multicolumn{1}{l|}{$98.93$}     & / & \multicolumn{1}{l|}{$99.50$}     & / & $64.44$    \\
    & GmapAD  & / & \multicolumn{1}{l|}{$50.00$} & / & \multicolumn{1}{l|}{$50.00$} & / & \multicolumn{1}{l|}{$50.00$} & / & \multicolumn{1}{l|}{$75.48$} & / & \multicolumn{1}{l|}{OOM}     & / & \multicolumn{1}{l|}{OOM}     & / & OOM     \\
    & RQGNN    & / & \multicolumn{1}{l|}{$\bf{98.79}$} & / & \multicolumn{1}{l|}{$97.98$} & / & \multicolumn{1}{l|}{$99.83$} & / & \multicolumn{1}{l|}{$96.41$} & / & \multicolumn{1}{l|}{$96.62$} & / & \multicolumn{1}{l|}{$95.57$} & / & 73.90 \\
    \midrule
    \multirow{2}{*}{Multi-task} 
    & GraphPrompt-U & $51.59$  & $46.85$  & $50.54$  & $48.67$  & $51.42$  & $49.38$  & $97.08$  & $68.23$ & $81.16$ & $83.88$ & $81.37$ & $6.16$ & $47.40$ & $50.81$\\
    & All-in-One-U  & $67.87$  & $3.21$   & $54.70$  & $19.42$  & $69.70$  & $45.89$  & $50.63$  & $48.98$ & OOT & OOT & OOT & OOT & OOT & OOT \\
    
     \midrule
    \multirow{2}{*}{\begin{tabular}[c]{@{}c@{}}UniGAD\\ (Ours)\end{tabular}} 
    & UniGAD - GCN   & \multicolumn{1}{l}{$\bf{99.75}$} & \multicolumn{1}{l|}{$94.29$} & \multicolumn{1}{l}{$\bf{99.60}$} & \multicolumn{1}{l|}{$\bf{99.67}$} & \multicolumn{1}{l}{$\bf{99.63}$} & \multicolumn{1}{l|}{$\bf{99.99}$} & \multicolumn{1}{l}{$99.50$} & \multicolumn{1}{l|}{$96.33$} & \multicolumn{1}{l}{$97.93$}      & $98.99$      & $98.11$  & $99.59$ & $95.57$ & $88.73$ \\
    &  UniGAD - BWG  & \multicolumn{1}{l}{$92.60$} & \multicolumn{1}{l|}{$68.74$} & \multicolumn{1}{l}{$93.30$} & \multicolumn{1}{l|}{$68.55$} & \multicolumn{1}{l}{$90.76$} & \multicolumn{1}{l|}{$56.01$} & \multicolumn{1}{l}{$\bf{99.54}$}  & \multicolumn{1}{l|}{$\bf{96.73}$}                     & \multicolumn{1}{l}{$\bf{99.99}$}     & $\bf{99.61}$      & $\bf{99.99}$ & $\bf{99.98}$ & $\bf{96.19}$ & $\bf{88.78}$ \\ 
    
    \bottomrule
    \end{tabular}
    }
    \end{table}

\noindent\textbf{Node-level and edge-level jointly.} We first evaluate the performance of unified training on node-level and edge-level data. We compare UniGAD against three groups of GNN models mentioned above: node-level models, edge-level models, and multi-task graph learning methods. Table \ref{NEfull} reports the AUROC of each model on six datasets, with the best result on each dataset highlighted in boldface. Overall, we find that UniGAD achieves state-of-the-art performance in nearly all scenarios. UniGAD outperforms single-level specialized models, indicating that unified training with UniGAD leverages information from other levels to enhance the performance of individual tasks. Multi-task approaches (GraphPrompt-U and All-in-One-U) tend to negatively impact multi-task performance, potentially because they are unable to effectively handle different types of anomaly label supervision. 
Meanwhile, UniGAD is designed for a multi-task setting, the performance on a single level might be slightly compromised to ensure the model performs well across all tasks in some datasets.

\noindent\textbf{Node-level and graph-level jointly.} We then evaluate the unified training of node-level and graph-level tasks under similar settings. Table \ref{NGfull} shows the results, and UniGAD achieves state-of-the-art performance in nearly all scenarios. Our observations are as follows. First, there is a multi-level synergy in UniGAD, where strong performance in one task benefits the performance of other tasks. For example, in MNIST-0 and MNIST-1, compared to other graph-level GAD methods, UniGAD significantly boosts graph-level performance by leveraging strong node-level results. Second, UniGAD performs better on large graphs, likely because graph structure plays a more significant role in smaller datasets. However, the backbones of UniGAD (GCN, BWGNN) are primarily node-level models, which may not effectively encode graph-level structural information. This limitation’s impact diminishes in large-scale graph datasets. Besides, methods like All-in-One-U often run out of time (OOT) with large datasets because they redundantly learn the same node representations across different subgraphs, making processing impractically slow, especially for large graph-level datasets like T-Group. UniGAD addresses this issue by using a shared GNN encoder across tasks, avoiding redundant learning and enhancing efficiency.

\subsection{The Transferability in Zero-Shot Learning (RQ2)}

\begin{table}[t]
\label{tab3}
    \caption{Zero-shot transferability (AUROC) at node and edge levels.}
    \label{NEzero}
    \centering
    \resizebox{\textwidth}{!}{
    \begin{tabular}{c|cc|cc|cc|cc|cc|cc|cc}
    \toprule
    \multirow{2}{*}{ \bf{Methods}} & \multicolumn{2}{c|}{\bf{Reddit}} & \multicolumn{2}{c|}{\bf{Weibo}} & \multicolumn{2}{c|}{\bf{Amazon}} & \multicolumn{2}{c|}{\bf{Yelp}} & \multicolumn{2}{c|}{\bf{Tolokers}} & \multicolumn{2}{c|}{\bf{Questions}} & \multicolumn{2}{c}{\bf{T-Finance}} \\ 
       & \bf{N$\rightarrow$E} & \bf{E$\rightarrow$N}  & \bf{N$\rightarrow$E} & \bf{E$\rightarrow$N} & \bf{N$\rightarrow$E} & \bf{E$\rightarrow$N}  & \bf{N$\rightarrow$E} & \bf{E$\rightarrow$N}  & \bf{N$\rightarrow$E} & \bf{E$\rightarrow$N}  & \bf{N$\rightarrow$E} & \bf{E$\rightarrow$N} & \bf{N$\rightarrow$E} & \bf{E$\rightarrow$N} \\ \midrule
    
     GraphPrompt-U & $54.06$  & $47.43$  & $57.03$  & $42.85$  & $49.76$			
 & $50.26$ & $49.97$ & $49.94$ & $48.56$  & $51.08$  & $54.26$  & $51.97$  & OOT & OOT \\
      All-in-One-U   & $49.23$  & $49.93$  & $52.22$  & $54.30$  & $52.61$ & $42.35$ & $49.48$ & $44.50$ & $48.34$  & $50.22$  & $49.83$  & $51.97$ & OOT & OOT 
       \\ 
     \midrule
     UniGAD - GCN & \multicolumn{1}{l}{$\bf{59.67}$} & \multicolumn{1}{l|}{$\bf{59.46}$} & \multicolumn{1}{l}{$\bf{98.31}$} & \multicolumn{1}{l|}{$\bf{98.59}$} & \multicolumn{1}{l}{$76.20$} & \multicolumn{1}{l|}{$82.38$} & \multicolumn{1}{l}{$58.28$} & \multicolumn{1}{l|}{$60.92$} & \multicolumn{1}{l}{$71.45$} & \multicolumn{1}{l|}{$73.35$} & $69.54$ & $\bf{65.37}$ & $91.63$ & $90.17$ \\
     UniGAD - BWG & \multicolumn{1}{l}{$53.32$} & \multicolumn{1}{l|}{$57.63$} & \multicolumn{1}{l}{$94.71$} & \multicolumn{1}{l|}{$96.87$} & \multicolumn{1}{l}{$\bf{82.64}$} & \multicolumn{1}{l|}{$\bf{96.41}$} & \multicolumn{1}{l}{$\bf{75.56}$} & \multicolumn{1}{l|}{$\bf{84.08}$} & $\bf{74.04}$ & \multicolumn{1}{l|}{$\bf{78.49}$} & $\bf{71.02}$ & $62.72$& $\bf{93.60}$ & $\bf{95.68}$ \\ 
    \bottomrule
    \end{tabular}
    }
\end{table}
\begin{table}[t]
    \label{tab4}
        \caption{Zero-shot transferability (AUROC) at node and graph levels.}
        \label{NGzero}
        \centering
        \resizebox{\textwidth}{!}{
        \begin{tabular}{c|cc|cc|cc|cc|cc|cc}
        \toprule
        \multirow{2}{*}{ \bf{Methods}}  & \multicolumn{2}{c|}{\bf{BM-MN}} & \multicolumn{2}{c|}{\bf{BM-MS}} & \multicolumn{2}{c|}{\bf{BM-MT}} & \multicolumn{2}{c|}{\bf{MUTAG}} & \multicolumn{2}{c|}{\bf{MNIST0}}  & \multicolumn{2}{c}{\bf{T-Group}} \\ 
             & \bf{N$\rightarrow$G} & \bf{G$\rightarrow$N}  & \bf{N$\rightarrow$G} & \bf{G$\rightarrow$N} & \bf{N$\rightarrow$G} & \bf{G$\rightarrow$N}  & \bf{N$\rightarrow$G} & \bf{G$\rightarrow$N}  & \bf{N$\rightarrow$G} & \bf{G$\rightarrow$N}  & \bf{N$\rightarrow$G} & \bf{G$\rightarrow$N} \\ \midrule
         GraphPrompt-U & $50.60$  & $51.57$  & $51.97$  & $46.95$  & $46.62$  & $48.06$  & $59.62$  & $64.26$ & $83.98$	 & $\bf{88.06}$ & $58.28$	 &  $58.35$ \\
        All-in-One-U    & $\bf{94.39}$  & $65.69$  & $52.63$  & $40.88$  & $44.86$  & $34.27$  & $61.63$  & $36.13$ & OOT & OOT &   OOT & OOT\\ 
            \midrule
        UniGAD - GCN & \multicolumn{1}{l}{$72.82$} & \multicolumn{1}{l|}{$\bf{87.63}$} & \multicolumn{1}{l}{$\bf{81.49}$} & \multicolumn{1}{l|}{$\bf{90.83}$} & \multicolumn{1}{l}{$\bf{62.85}$} & \multicolumn{1}{l|}{$\bf{79.26}$} & $\bf{72.79}$ & \multicolumn{1}{l|}{$\bf{88.53}$} & \multicolumn{1}{l}{$\bf{85.24}$} & \multicolumn{1}{l|}{$70.57$} &  \multicolumn{1}{l}{$\bf{86.86}$} &  \multicolumn{1}{l}{$\bf{75.89}$}  \\
        UniGAD - BWG & \multicolumn{1}{l}{$64.61$} & \multicolumn{1}{l|}{$57.56$} & \multicolumn{1}{l}{$65.33$} & \multicolumn{1}{l|}{$51.34$} & \multicolumn{1}{l}{$55.78$} & \multicolumn{1}{l|}{$53.41$} & \multicolumn{1}{l}{$66.92$} & \multicolumn{1}{l|}{$87.03$} & \multicolumn{1}{l}{$74.23$} & \multicolumn{1}{l|}{$63.70$} &  \multicolumn{1}{l}{$86.81$} & \multicolumn{1}{l}{$64.81$} \\ 
        \bottomrule
        \end{tabular}
        }
    \end{table}

To assess the transfer capability of UniGAD, we explore zero-shot learning scenarios where labels for a given level have never been exposed during training, as shown in Tables \ref{NEzero} and \ref{NGzero}. In these experiments, UniGAD is trained solely with labels from alternative levels. The notation $\mathbf{N \rightarrow E}$ indicates using node labels to infer edge labels, with analogous notations for other label transfers. Our findings indicate that in zero-shot scenarios, UniGAD outperforms existing multi-task prompt learning methods. Moreover, the classification performance of UniGAD under zero-shot transfer learning even surpasses some of the leading baselines in supervised settings on Yelp and BM-MS. It highlights the superior transfer capability of UniGAD across various GAD tasks.

\subsection{Ablation Study (RQ3)}

\noindent

\begin{wraptable}{R}{8.2cm}
    \centering
    \vspace{-13pt}
    \scriptsize
    \setlength{\tabcolsep}{3pt}
    \caption{Performance of UniGAD and its variants.}

    \begin{tabular}{c|cccc|cccc}
        \toprule
         & \multicolumn{4}{c|}{BM-MS} & \multicolumn{4}{c}{Reddit} \\  \midrule
        Metrics & \multicolumn{2}{c|}{AUROC} & \multicolumn{2}{c|}{Macro F1} & \multicolumn{2}{c|}{AUROC} & \multicolumn{2}{c}{Macro F1} \\
        Task-level & node    & \multicolumn{1}{c|}{graph} & node & graph & node & \multicolumn{1}{c|}{edge} & node & edge \\  \midrule
        w/o GS.    & $97.13$ & \multicolumn{1}{c|}{$98.99$} & $80.35$ & $95.79$ & $68.69$ & \multicolumn{1}{c|}{$66.06$} & $53.83$ & $52.78$ \\ 
        w 2hop.    & $97.49$ & \multicolumn{1}{c|}{$99.94$} & $67.29$ & $84.20$ & $67.53$ & \multicolumn{1}{c|}{$63.62$} & $51.77$ & $50.69$ \\
        w RS.      & $93.85$ & \multicolumn{1}{c|}{$84.92$} & $85.88$ & $72.21$ & $65.32$ & \multicolumn{1}{c|}{$61.85$} & $52.32$ & $51.03$ \\
        \midrule
        w/o ST.    & $99.94$ & \multicolumn{1}{c|}{$95.51$} & $99.47$ & $84.91$ & $67.74$ & \multicolumn{1}{c|}{$65.92$} & $54.35$ & $52.52$ \\ 
        \midrule
        UniGAD     & $99.60$ & \multicolumn{1}{c|}{$99.67$} & $99.57$ & $95.86$ & $71.65$ & \multicolumn{1}{c|}{$65.46$} & $56.70$ & $53.80$ \\  
        
        \bottomrule
        \end{tabular}     
        \label{abla}
\end{wraptable} 

To investigate the contribution of each module in UniGAD, we present the ablation study results in Table \ref{abla}. For the sampler module, we compare the results without subgraph sampling (w/o GS.), using a simple sampler with all 2-hop neighbors (w 2hop.), and using random sampling (w RS.). For the GraphStitch module, we replace it with a unified MLP (w/o ST.). The results indicate that both the subgraph sampler (SG.) and the GraphStitch (ST.) modules enhance the overall performance of UniGAD. Additionally, inappropriate subgraph sampling may perform worse than no subgraph sampling, likely due to the loss of anomalous information during the sampling process.

\subsection{Efficiency Analysis (RQ4)}
we conduct a comprehensive evaluation of both time and space efficiency on the large-scale, real-world T-Group dataset. 
To provide a more straightforward comparison between single-task and multi-task baselines, we calculate the average, minimum, and maximum for combinations of single-task node-level and graph-level models, and compare these with multi-task models. The results, as shown in Fig. \ref{time} (a), indicate that in terms of execution time, our method is slower than the combination of the fastest single-level models but faster than the average of the combination.
Regarding peak memory usage, Fig. \ref{time} (b) demonstrates that graph-level models consume significantly more memory than node-level models. Our method maintains memory consumption comparable to node-level models and substantially lower than both graph-level GAD models and prompt-based methods.
\begin{figure}[t]
    \centering
    \begin{subfigure}[t]{0.45\textwidth}
        \centering
        \includegraphics[height=2.2in]{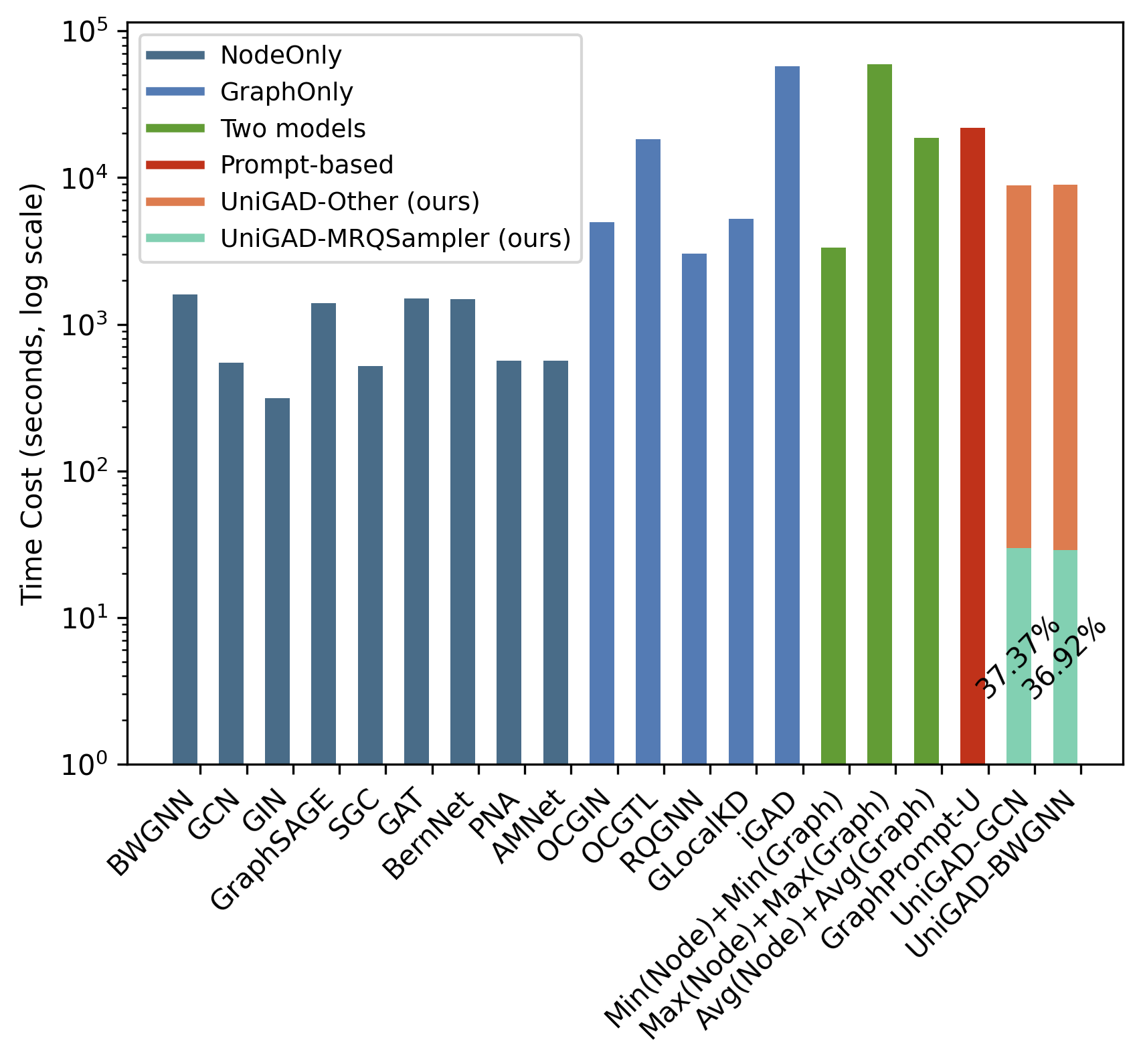}
        \caption{Comparing the execution time.}
    \end{subfigure}%
    \begin{subfigure}[t]{0.45\textwidth}
        \centering
        \includegraphics[height=2.2in]{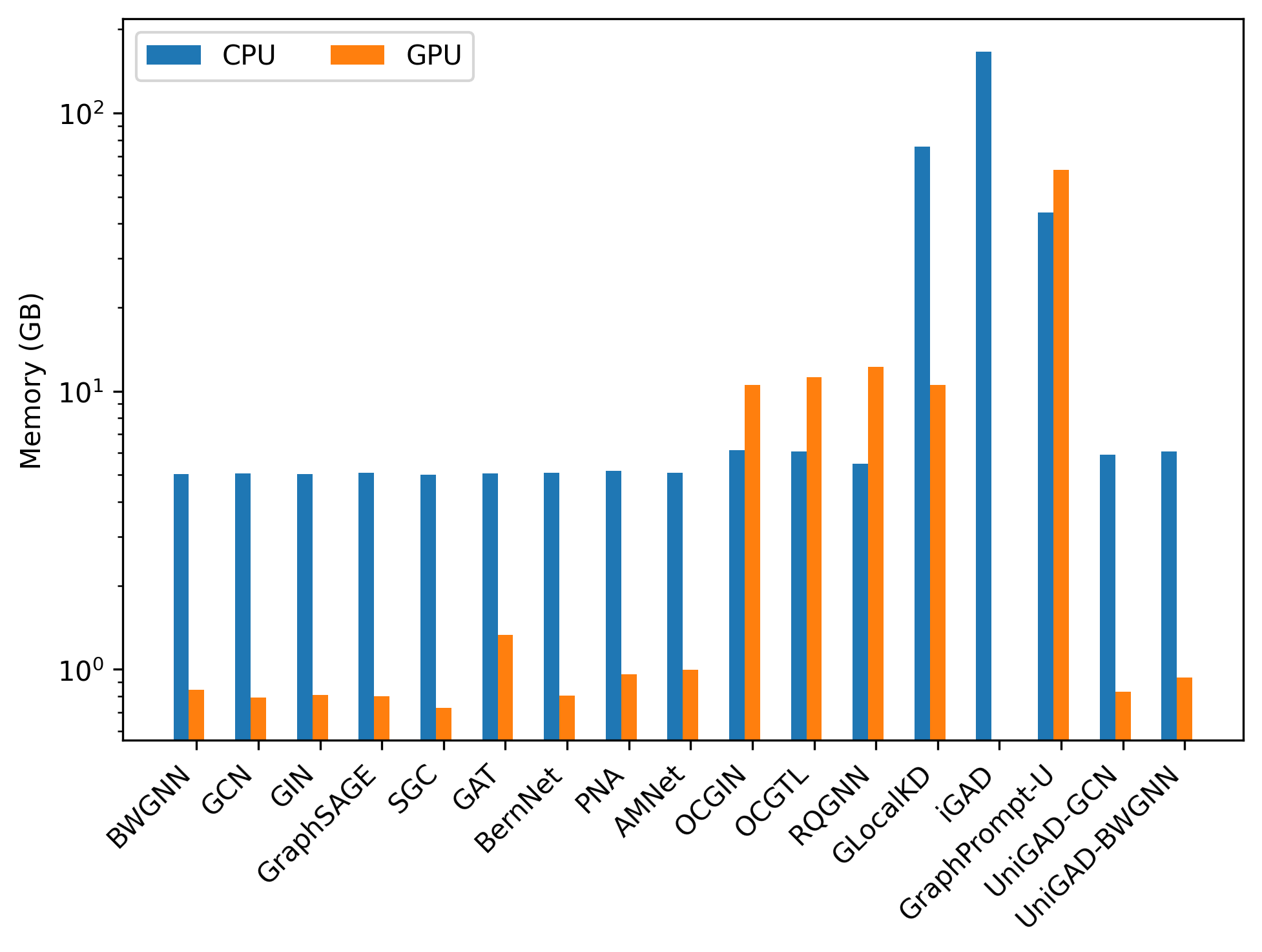}
        \caption{Comparing the peak memory usage.}
    \end{subfigure}
    \caption{The evaluation of time and space efficiency metrics. We highlight the percentage of total execution time spent by MRQSampler.}
    \label{time}  
\end{figure}

\section{Conclusion}
This paper presents UniGAD, a unified graph anomaly detection framework that jointly addresses anomalies at the node, edge, and graph levels. The model integrates two novel components: the MRQSampler and the GraphStitch network. MRQSampler maximizes spectral energy to ensure subgraphs capture critical anomaly information, addressing the challenge of unifying different graph object formats. The GraphStitch Network unifies multi-level training by using identical networks for nodes, edges, and graphs, facilitated by the GraphStitch Unit for effective information sharing. Our thorough evaluations across 14 GAD datasets, including two real-world large-scale datasets (T-Finance and T-Group), and comparisons with 17 graph learning methods show that UniGAD not only surpasses existing models in various tasks but also exhibits strong zero-shot transferability capabilities.
A limitation of our paper is that the GNN encoder primarily focuses on node-level embeddings, which may result in lost information about the graph structure. We leave the exploration of multi-level tasks pre-training in the future works.

\begin{ack}
Y. Lin and H. Zhao were supported by the Beijing Natural Science Foundation under Grant IS24036. J. Li was supported by NSFC Grant No. 62206067 and Guangzhou-HKUST(GZ) Joint Funding Scheme 2023A03J0673. Y.Yao was in part supported by the HKRGC GRF-16308321 and the NSFC/RGC Joint Research Scheme Grant N\_HKUST635/20. In addition, Y. Lin was also awarded a Tsinghua Scholarship for Overseas Graduate Studies at the Hong Kong University of Science and Technology. 
\end{ack}

\bibliography{refs}
\bibliographystyle{plain}

\newpage
\appendix

\renewcommand{\theequation}{A.\arabic{equation}}
\setcounter{equation}{0}

\section{Proofs}

\subsection{The proof of Theorem 1}
\label{proof1}
\begin{proof}
For a new node $v_{new}$, let $\mathcal{S}'$ be the subgraph after the addition of $v_{new}$. Based on the fact that the definition of Rayleigh quotient $RQ(\mathcal{S})=\frac{\sum_{(p, q) \in \mathcal{E}_{\mathcal{S}}} (x_p-x_q)^2}{\sum_{p \in \mathcal{S}} x_p^2}$, it need to satisfy the following condition in order to increase the Rayleigh quotient $RQ(\mathcal{S}') > RQ(\mathcal{S})$:
\begin{equation}
\label{eq:A1}
    \frac{\sum_{(p,q) \in \mathcal{E}} (x_p-x_q)^2 +\sum_{v_r \in \mathcal{S}}(x_{new}-x_r)^2}{\sum_{v_r \in \mathcal{S}}x_r^2 + x_{new}^2} > \frac{\sum_{(p,q) \in \mathcal{E}} (x_p-x_q)^2 }{\sum_{v_r \in \mathcal{S}}x_r^2},
\end{equation}
where $\sum_{v_r \in \mathcal{S}}(x_{new}-x_r)^2$ represents the sum of the feature difference between the new node $v_{new}$ and the connecting edge of the node $v$ in the subgraph $\mathcal{S}$. It is worth noting that these edges are present in the original graph $\mathcal{G}$.
Since both the numerator and denominator are positive numbers, the Eq. (\ref{eq:A1}) can be transformed into:
\begin{equation}
\label{eq:A2}
    \left[\sum_{(p,q) \in \mathcal{E}} (x_p-x_q)^2  +\sum_{p \in \mathcal{S}}(x_{new}-x_p)^2 \right] \sum_{v_r \in \mathcal{S}}x_r^2
    >\sum_{(p,q) \in \mathcal{E}} (x_p-x_q)^2 ( \sum_{v_r \in \mathcal{S}}x_r^2 + x_{new}^2) ,
\end{equation}
which can be obviously simplified to:
\begin{equation}
\label{eq:A3}
    \sum_{v_r \in \mathcal{S}}(x_{new}-x_r)^2   \sum_{v_r \in \mathcal{S}}x_r^2
    > x_{new}^2\sum_{(p,q) \in \mathcal{E}} (x_p-x_q)^2 ,
\end{equation}
We move the term with the new node to the same side of the equation and rearrange the Eq. (\ref{eq:A3}), and we obtain:
\begin{equation}
\label{eq:4}
 \frac{\sum_{ v_r \in \mathcal{S}} (x_{new}-x_r)^2}{x_{new}^2} > \frac{\sum_{(p, q) \in \mathcal{E}_{\mathcal{S}}} (x_p-x_q)^2}{\sum_{p \in \mathcal{S}} x_p^2}.
\end{equation}
Note that $RQ(\mathcal{S})=\frac{\sum_{(p, q) \in \mathcal{E}_{\mathcal{S}}} (x_p-x_q)^2}{\sum_{p \in \mathcal{S}} x_p^2}$ , we denote $\Delta(v_{new}) = \frac{\sum_{ v_r \in \mathcal{S}} (x_{new}-x_r)^2}{x_{new}^2}$ and we then obtain the Theorem \ref{theorem1} in Section \ref{MRQSampler}. \qedhere
\end{proof}

\subsection{The proof of Corollary 1}
\label{proof1c}
\begin{proof}
Similar to the proof of Theorem \ref{theorem1}, for a new nodeset $v_{{new}_i} \in \mathcal{V}_{new}$, let $\mathcal{S}'$ be the subgraph after the addition of $\mathcal{V}_{new}$ and it also needs to satisfy  $RQ(\mathcal{S}') > RQ(\mathcal{S})$, which is expanded as:
\begin{equation}
\label{eq:A5}
\begin{aligned}
\! \! \! \! \! \! \! \! \! 
\frac{\sum_{(p,q) \in \mathcal{E}} (x_p-x_q)^2 +\sum_{v_{{new}_i} \in \mathcal{V}_{{new}_i}}\sum_{v_r \in \mathcal{S}}(x_{{new}_i}-x_r)^2 + \sum_{(i,j) \in \mathcal{E}_{new}} (x_{{new}_i}-x_{{new}_j})^2}{\sum_{v_r \in \mathcal{S}}x_r^2 + \sum_{v_{{new}_i} \in \mathcal{V}_{new} }x_{{new}_i}^2} \\ > \frac{\sum_{ (p,q) \in \mathcal{E}} (x_p-x_q)^2 }{\sum_{ v_r \in \mathcal{S}}x_r^2},
\end{aligned}
\end{equation}
where $\sum_{v_{{new}_i} \in \mathcal{V}_{{new}_i}}\sum_{v_r \in \mathcal{S}}(x_{{new}_i}-x_r)^2$ represents the sum of the feature difference between the newly added nodeset $\mathcal{V}_{new}$ and the connecting edge of the subgraph $\mathcal{S}$. $\sum_{v_{{new}_i} \in \mathcal{V}_{new} }x_{{new}_i}^2$ represents the internal sum of the newly added nodeset $\mathcal{V}_{new}$. Similar to the proof of Theorem \ref{theorem1}, this formula can be transformed into:
\begin{equation}
\label{eq:A6}
\! \! \! \! \! \! \! \! \! \! \! \! 
\left[\sum_{v_{{new}_i} \in \mathcal{V}_{{new}_i}}\sum_{v_r \in \mathcal{S}}(x_{{new}_i}-x_r)^2  + \! \! \!\!\!  \sum_{(i,j) \in \mathcal{E}_{new}} (x_{{new}_i}-x_{{new}_j})^2\right] \sum_{ v_r \in \mathcal{S}}x_r^2  > \!\!\!\! \sum_{ (p,q) \in \mathcal{E}} (x_p-x_q)^2 \!\!\!\! { \sum_{v_{{new}_i} \in \mathcal{V}_{new} } \!\!\!\! x_{{new}_i}^2},
\end{equation}
Rearranging the Eq. (\ref{eq:A6}), we get:
\begin{equation}
\label{eq:A7}
 \frac{\sum_{ v_{new} \in \mathcal{V}_{new}} \sum_{ v_r \in \mathcal{S}} (x_{new}-x_r)^2+\sum_{(i, j) \in \mathcal{E}_{\mathcal{V}_{new}}} (x_{new_i}-x_{new_j})^2}{\sum_{ v_{new} \in \mathcal{V}_{new}} x_{new}^2} > \frac{\sum_{ (p,q) \in \mathcal{E}} (x_p-x_q)^2 }{\sum_{ v_r \in \mathcal{S}}x_r^2}.
\end{equation}
which is the same as Corollary \ref{corollary11} in Section \ref{MRQSampler}.

\end{proof}

\subsection{The proof of Theorem 2}
\label{proof2}
\begin{proof}
We define the nodeset $\mathcal{V}_{new}^*$ has the highest $\Delta_{max}(\mathcal{V}_{new})$ and $\Delta_{max}(\mathcal{V}_{new}) > RQ(\mathcal{S})$.
To prove that the $\mathcal{V}_{new}^*$ is contained in the optimal subgraph $\mathcal{S}^*$, we give the proof by contradiction. Assume that the negation of the statement is true, so there does not exist $\mathcal{V}_{new}^*$ in $\mathcal{S}^*$. We will discuss the issues based on two scenarios.

In the first scenario, we assume that the current subgraph $\mathcal{S}$ is already the optimal solution. According to Corollary \ref{corollary11}, we find that adding $\mathcal{V}_{new}^*$ can increase $RQ(S)$ since it satisfies the condition $\Delta_{max}(\mathcal{V}_{new}) > RQ(\mathcal{S})$. Therefore, it is obvious that the current set $\mathcal{S}$ is not the optimal solution.

In the other scenario, we assume that there is another nodeset $\mathcal{V}_{new}'$ ( $\mathcal{V}_{new}' \cap \mathcal{S}^* = \emptyset$), which together with the current subgraph $\mathcal{S}+\mathcal{V}_{new}'$ forms the optimal solution. According to the corollary \ref{corollary11}, we have
\begin{equation}
\label{eq:A8}
\Delta_{max}(\mathcal{V}_{new}^*) = \frac{\sum_{ \mathcal{V}_{new}^*} \sum_{  \mathcal{S}} (x_{new}^*-x_r)^2+\sum_{\mathcal{E}_{\mathcal{V}_{new}^*}} (x_{new_i}^*-x_{new_j}^*)^2}{\sum_{ \mathcal{V}_{new}^*} {x_{new}^*}^2},
\end{equation}
and it satisfies:
\begin{equation}
\label{eq:A9}
\begin{cases}
\Delta_{max}(\mathcal{V}_{new}^*)  > \frac{\sum_{ (p,q) \in \mathcal{E}} (x_p-x_q)^2 }{\sum_{ v_r \in \mathcal{S}}x_r^2}, \\
\Delta_{max}(\mathcal{V}_{new}^*)  >   \frac{\sum_{ \mathcal{V}_{new}'} \sum_{  \mathcal{S}} (x_{new}'-x_r)^2+\sum_{\mathcal{E}_{\mathcal{V}_{new}'}} (x_{new_i'}-x_{new_j'})^2}{\sum_{ \mathcal{V}_{new}'} {x_{new}'}^2}, \forall_{\mathcal{V}_{new}' \subseteq \mathcal{G}-\mathcal{S}}.
\end{cases}
\end{equation}
To continue with the proof, we present a useful inequality first.
\begin{lemma}[Dan's Favorite Inequality]
\label{lemma3}
Let $a_1, ..., a_n$ and $b_1, ..., b_n$ be positive numbers. Then
    \begin{equation}
        \min _i \frac{a_i}{b_i} \leq \frac{\sum_i a_i}{\sum_i b_i} \leq \max _i \frac{a_i}{b_i} .
    \end{equation}  
\end{lemma}
\begin{proof}
Here we give a classical proof, we have
\begin{equation}
\sum_i a_i=\sum_i b_i\left(\frac{a_i}{b_i}\right) \leq \sum_i b_i\left(\max _j \frac{a_j}{b_j}\right)=\left(\max _j \frac{a_j}{b_j}\right) \sum_i b_i ,
\end{equation}
So, 
\begin{equation}
\frac{\sum a_i}{\sum b_i} \leq \max _j \frac{a_j}{b_j},
\end{equation}
One can similarly prove
\begin{equation}
\frac{\sum a_i}{\sum b_i} \geq \min _j \frac{a_j}{b_j}.
\end{equation}
\end{proof}
Combining Lemma \ref{lemma3} and Eq. (\ref{eq:A9}), we obtain the following inequality.
\begin{equation}
\label{eq:A14}
\begin{aligned}
 \!  \!\!\!\!\!\!\!\!\!\!\!\!\!\!\!\! \!\!\!\!\!\! &\Delta_{max}(\mathcal{V}_{new}^*)  > \\
    &\frac{\sum_{ (p,q) \in \mathcal{E}} (x_p-x_q)^2 +\sum_{ \mathcal{V}_{new}'} \sum_{  \mathcal{S}} (x_{new}'-x_r)^2+\sum_{\mathcal{E}_{\mathcal{V}_{new}'}} (x_{new_i}'-x_{new_j}')^2}{\sum_{ v_r \in \mathcal{S}}x_r^2+\sum_{ \mathcal{V}_{new}'} {x_{new}'}^2}, \forall_{\mathcal{V}_{new}' \subseteq \mathcal{G}-\mathcal{S}}. 
\end{aligned}
\end{equation}
Analyzing the above equation reveals that the right side of the formula is $RQ(\mathcal{V}_{new}'+\mathcal{S})$. That is, for any $\mathcal{V}_{new}'$, adding $\mathcal{V}_{new}^*$ still makes $RQ(\mathcal{V}_{new}'+\mathcal{S})$ increasing according to the Corollary \ref{corollary11}, which contradicts the assumption that $RQ(\mathcal{V}_{new}'+\mathcal{S})$ is the optimal solution.

\begin{equation}
\label{eq:5}
\Delta_{max}(\mathcal{V}_{new}) = \max_{\mathcal{V}_{new} \subseteq \mathcal{G}-\mathcal{S}}\Delta(\mathcal{V}_{new}), \ \text{and} \
\Delta_{max}(\mathcal{V}_{new}) > RQ(\mathcal{S}).
\end{equation}

 However, identifying the maximum $\Delta_{max}(\mathcal{V}_{new})$ from the $\mathcal{V}_{new} \subseteq \mathcal{G}-\mathcal{S}$ is still a NP-hard problem. We consider relaxing any nodesets  $\mathcal{V}_{new}'$ to any \textbf{connected} nodesets $\mathcal{V}_{new}^{c}$. Any nodesets can be decomposed into several disconnected smaller nodesets, that is, $ \mathcal{V}_{new}' =\mathcal{V}_{new}^{c1} \cup \mathcal{V}_{new}^{c2}\cup \ldots $. Since there are no edges connecting these nodesets, the following decomposition formula can be derived. 

\begin{equation}
    \begin{cases}
        \sum_{ \mathcal{V}_{new}'} \sum_{  \mathcal{S}} (x_{new}'-x_r)^2 = \sum_{ \mathcal{V}_{new}^{c1}} \sum_{  \mathcal{S}} (x_{new}^{c1}-x_r)^2 + \sum_{ \mathcal{V}_{new}^{c2}} \sum_{  \mathcal{S}} (x_{new}^{c2}-x_r)^2+ \ldots, \\
        \sum_{\mathcal{E}_{\mathcal{V}_{new}'}} (x_{new_i}'-x_{new_j}')^2 =  \sum_{\mathcal{E}_{\mathcal{V}_{new}^{c1}}} (x_{new_i}^{c1}-x_{new_j}^{c1})^2 + \sum_{\mathcal{E}_{\mathcal{V}_{new}^{c2}}} (x_{new_i}^{c2}-x_{new_j}^{c2})^2 + \ldots, \\
        \sum_{ \mathcal{V}_{new}'} {x_{new}'}^2 =  \sum_{ \mathcal{V}_{new}^{c1}} {x_{new}^{c1}}^2 + \sum_{ \mathcal{V}_{new}^{c2}} {x_{new}^{c2}}^2 +\ldots.
    \end{cases}
\end{equation}
Considering the condition that maximizes the Rayleigh quotient of any connected $\mathcal{V}_{new}^{ci}$,
\begin{equation}
    \Delta_{max}(\mathcal{V}_{new}^*) > \Delta_{max}(\mathcal{V}_{new}^{ci}) = \frac{\sum_{ \mathcal{V}_{new}^{ci}} \sum_{  \mathcal{S}} (x_{new}^{ci}-x_r)^2 + \sum_{\mathcal{E}_{\mathcal{V}_{new}^{ci}}} (x_{new_i}^{ci}-x_{new_j}^{ci})^2}{\sum_{ \mathcal{V}_{new}^{ci}} {x_{new}^{ci}}^2}.
\end{equation}
According to Lemma \ref{lemma3}, Eq. (\ref{eq:A14}) is still satisfied. Therefore, we derive that $\mathcal{V}_{new}^*$ is contained in the optimal solution.

\end{proof}

\section{The Pseudocode of MRQSampler Algorithm}
\label{pseudocode}

We give the pseudocode of MRQSampler in Algorithm \ref{alg:sampling}, which illustrates the algorithm for finding the subgraph with the target node that maximizes the Rayleigh Quotient.
In section \ref{MRQSampler}, we give a diagram of the sampling range of 2-hop and 1-hop cases. 
For the completeness of the theory, we give the complete algorithm for arbitrary $k$-hop cases in the pseudocode form.

For node $r$, we focus on the $k$-hop spanning tree $T$ with $r$ as the root node. And for any node $v$ in $T$ except for the root $r$, $\Delta_{max}[v]$ is defined as:

\begin{equation}
    \Delta_{max}[v] := \max_{\mathcal{S} \subseteq \mathcal{T}_{v}} \frac{\left( x_{i} - x_{p} \right)^{2} + \sum_{(i,j)\subseteq \mathcal{E}_{\mathcal{S}}} \left(x_{i} - x_{j} \right)^{2} }{ \sum_{i\subseteq \mathcal{S}} x_{i}^{2} }.
\end{equation}
where $\mathcal{S} \subseteq \mathcal{T}_{v}$ are the connected subgraphs of the subtree $T_v$ with $v$ as the root node, and $p$ is the parent node of the node $v$.
As described in the Section \ref{MRQSampler}, we break the computation into two steps: 

\begin{itemize}
    \item \textbf{Stage 1}: Compute and store the maximum $\Delta_{max}[v]$ for subtrees rooted with each node $v$ except for the root $r$, which is performed by recursively calling the function \texttt{GetMaxDeltas} in Algorithm \ref{alg:sampling}. 
    \item \textbf{Stage 2}: Use these memorized results to compute the optimal Rayleigh Quotient $RQ$ and its corresponding subgraph, which is performed by the function \texttt{MRQSampler}.
\end{itemize}

In \textbf{Stage 1}, the first thing we need to know is how we get $\Delta_{max}[v]$.
Similar to the analysis of the Theorem \ref{corollary2}, we can also obtain the condition that the nodeset is in the final optimal subgraph with largest $\Delta_{max}[v]$:
\begin{equation}
\label{eq:new10}
\Delta_{max}[v_{new}] = \max_{\{\tilde{v}_{new}\}}\Delta(\tilde{\mathcal{V}}_{new}), \ \text{and} \
\Delta_{max}[v_{new}] > \Delta[v].
\end{equation}

This process is similar to finding the maximum $RQ$. In other words, we keep retrieving the un-selected descendants with maximum $\Delta_{max}[v_{new}]$, and then check whether its $\Delta_{max}[v_{new}]$ exceeds the current $\Delta[v]$. If it does, the inclusion of the optimal subgraph with it can increase the current $\Delta[v]$, otherwise, it can no longer be increased by adding any descendants and the maximum $\Delta_{max}[v]$ is reached.

\begin{algorithm}[H]
    \caption{MRQSampler}
        \label{alg:sampling}
    \begin{algorithmic}[1]
        \Statex \textbf{Globals:} $r:$ the original root of the tree; $i:$ an arbitrary node; $x[i]:$ the node $i$'s feature;  $T_r[i]:$ an array that stores the child nodes of node $i$ in the tree; $\delta[i]\leftarrow\{\Delta_{max}[i], a_i, b_i, N, I \}$: an array of structures that stores the maximum $\Delta_{max}[i]$ achievable by any connected subgraph $\mathcal{V}$ containing the node $i$ within the subtree rooted at $i$ and $a_i, b_i$ stores the numerator and denominator of $\Delta_{max}[i]$. $N$ is the optimal selected nodes, $I$ is the inferior candidates,
        \Statex
        \\
        \texttt{\# The function for computing $\Delta_{max}[i]$ in STAGE I}
        \State \textbf{Input:} $v \gets$ root of the current subtree; $p \gets$ parent of $v$
        \State \textbf{Output:} $\delta[i] \gets$ structure array with $\Delta_{max}[i]$ and correlated variables
        \Function{\texttt{GetMaxDeltas}}{$v, p$} 
        \Let{$N, I, U$}{$ \{ \}$} \Comment{Optimal selected nodes, Inferior candidates, Un-selected children of $v$}
        \Let{$Q$}{\texttt{SortedSet()}} \Comment{Candidates queue}
            \Let{$a_v$}{$ (x[v]-x[p])^{2} $} \Comment{Initialize the numerator of the maximum $\Delta_{max}[v]$}
        \Let{$b_v$}{$ x[v]^{2} $} \Comment{Initialize the denominator of the maximum $\Delta_{max}[v]$}
        \Let{$\Delta_{max}[v]$}{$ a_v / b_v $} \Comment{Initialize the maximum $\Delta_{max}[v]$ for current sub-tree}
        \For{$c \textrm{ in } T[v]$}
            \Let{$\delta[c]$}{\texttt{GetMaxDeltas}($c$, $v$)} \Comment{Result of the subtree rooted with child $c$}
            \State $Q$.\texttt{insert([$c, \delta[c]$])}
            \State $U$.\texttt{insert($c$)}
        \EndFor
        \While{$Q$.\texttt{size()} $\neq 0$}
            \Let{$c, \delta[c]$}{$Q$.\texttt{pop\_largest()}} \Comment{Retrieve the candidate $c$ with $\Delta_{max}[c]$ and structure $\delta[c]$ }
            \If{$ \Delta_{max}[v] > (a_{v} + \delta[c].a_{c}) / (b_{v} + \delta[c].b_{c}) $} \Comment{Optimization criterion of $\Delta_{max}[v]$}
                \State Break
            \EndIf
            \State $U$.\texttt{remove\_if\_exist($c$)}
            \Let{$a_{v}$}{$a_{v} + \delta[c].a_{c}$} \Comment{Update the maximum $\Delta_{max}[v]$}
            \Let{$b_{v}$}{$b_{v} + \delta[c].b_{c}$}
            \Let{$\Delta_{max}[v]$}{$a_v / b_v$} 
            \State $Q$.\texttt{insert($\delta[c].I$)} \Comment{Activate the inferior candidates}
            \State $N$.\texttt{insert($\delta[c].N$)}
        \EndWhile
        \Let{$I$}{$Q$} \Comment{The remaining candidates are the inferior ones}
        \State $I$.\texttt{insert($U$)} \Comment{Add the un-selected children to the inferior set}
        \Let{$\delta[v]$}{$\{\Delta_{max}[v], a_v, b_v, N, I \}$} \Comment{Memorise the results}
        \State \Return{$\delta[v]$}
        \EndFunction
        \\
        \\
        \texttt{\# The main function of MRQSampler in STAGE II}
    \State \textbf{Input:} $r \gets$ the original root of the tree (target sampling node)
    \State \textbf{Output:} $RQ_{max} \gets$ maximum $RQ$; $N \gets$ optimal sampling nodeset 
    \Function{\texttt{MRQSampler}}{$r$} 
        \Let{$a_{RQ}$}{$0$}  \Comment{Initialize the numerator of the $RQ_{max}$} 
        \Let{$b_{RQ}$}{$x[r]^{2}$}\Comment{Initialize the denominator of the $RQ_{max}$}
        \Let{$RQ_{max}$}{$a_{RQ}/b_{RQ}$}
        \Let{$Q$}{\texttt{SortedSet()}}
        \For{$c \textrm{ in } T[r]$}
            \Let{$\delta[c]$}{\texttt{GetMaxDeltas}(c, $r$)}\Comment{Recursively calculate the $\Delta_{max}$ in Stage 1}
            \State $Q$.\texttt{insert([$c, \delta[c]$])}
        \EndFor
        \While{$Q$.\texttt{size()} $\neq 0$}
            \Let{$c, \delta[c]$}{$Q$.\texttt{pop\_largest()}}
            \If{$ RQ_{max} > (a_{RQ} + \delta[c].a_{c}) / (b_{RQ} + \delta[c].b_{c}) $} \Comment{Optimization criterion of the RQ}
                \State Break
            \EndIf
            \Let{$a_{RQ}$}{$a_{RQ} + \delta[c].a_{c}$} \Comment{Update the result}
            \Let{$b_{RQ}$}{$b_{RQ} + \delta[c].b_{c}$}
            \Let{$RQ_{max}$}{$a_{RQ} / b_{RQ}$} 
            \State $Q$.\texttt{insert($\delta[c].I$)} \Comment{Activate the inferior candidates}
            \State $N$.\texttt{insert($\delta[c].N$)} \Comment{Update the selected nodeset}
        \EndWhile
        \State \Return{$\{RQ_{max}, N\}$}
    \EndFunction
    \end{algorithmic}
\end{algorithm}

For ease of computation, we store the optimal $\Delta_{max}[v]$ by its numerator $a_v$ and denominator $b_v$ (line 7-9). 
Next, we recursively calculate the result for each subtree rooted by every child $c$ of the current root $v$ (line 11).
To simplify complexity in the following steps, we store these results in a sorted container (e.g. binary search tree) $Q$ (line 12).
Next, we keep retrieving the subgraph with the highest $\Delta_{max}[c]$ from $Q$ and compute whether the $\Delta_{max}[c]$ increases after adding it to the current solution (line 15-17).
According to Eq. (\ref{eq:new10}), we obtain the optimal $\Delta_{max}[v]$ for the current subtree with root $v$ when the candidate cannot make $\Delta_{max}[v]$ larger.
Moreover, sets from subtrees that are not optimal for $v$ may still be selected at higher levels.
Therefore, we also need to keep track of those inferior subtrees and re-consider them when other subtrees that connect to them are merged into the solution (line 24-25).
Note that subtrees in $I$ are only considered as candidates when the optimal subgraph with root $v$ is selected at a higher layer (the ``activation'' in line 22). 

In \textbf{Stage 2}, the overall routine for obtaining $RQ_{max}$ is very similar to the one in \texttt{GetMaxDeltas}, except that the initial value is set to $RQ_{max} = \frac{0}{x[r]^2}$ since the root $r$ has no parent node. In other words, the algorithmic logic of the two functions in Stage 1 and Stage 2 is similar. Stage 2 can be regarded as a special case of Stage 1 without a parent node, utilizing the implementation of memoization from Stage 1.

Assuming that the k-hop spanning tree $T$ has $K$ nodes, the time complexity of Algorithm $\mathcal{O}(KlogK)$, since we will at worst examine each edge and sort them.
Notice that the computation is irrelevant between different nodes, it can be further accelerated by simultaneously processing multiple nodes.
In practice, we observe that the optimal choice of k-hop is typically <= 2, and thus the recursive computation can be unrolled thus further improving the efficiency.

\section{Implementation Details}
\label{expdetails}

\paragraph{Node-level Baselines.} \textbf{GCN} (Graph Convolutional Network \cite{GCN}) leverages convolution operations to propagate information from nodes to their neighbors. \textbf{SGC} (Simple Graph Convolution \cite{wu2019simplifying}) further simplifies GCN by removing non-linearities and collapsing weight matrices between consecutive layers to improve efficiency. \textbf{GIN} (Graph Isomorphism Network \cite{GIN}) captures graph structures by generating identical embeddings for structurally identical graphs, ensuring invariance to node label permutations. \textbf{GraphSAGE} (Graph Sample and AggregatE \cite{GraphSAGE}) generates node embeddings through sampling and aggregating features from local neighborhoods, supporting inductive learning. \textbf{GAT} (Graph Attention Networks \cite{GAT}) incorporates an attention mechanism to assign varying importance levels to different nodes during neighborhood aggregation, focusing on the most informative parts. \textbf{PNA} (Principle Neighbor Aggregation \cite{PNA}) combines multiple aggregators with degree-scalers for effective neighborhood aggregation. \textbf{AMNet} (Adaptive Multi-frequency Graph Neural Network \cite{AMNet}) captures both low and high-frequency signals by stacking multiple BernNets \cite{BernNet}, adaptively combining signals of different frequencies. \textbf{BWGNN} (Beta Wavelet Graph Neural Network \cite{BWGNN}) employs the Beta kernel to tackle higher frequency anomalies with flexible band-pass filters.

\paragraph{Graph-level Baselines.} \textbf{OCGIN} \cite{OCGIN} is a one-class graph-level anomaly detector based on a graph isomorphism network that addresses performance fluctuations in general graph classification methods. \textbf{OCGTL} \cite{OCGTL} combines deep one-class classification with graph transformation learning. \textbf{GlocalKD} learns rich global and local normal pattern information by joint distillation of graph and node representations. \textbf{iGAD} \cite{iGAD} employs an attribute-aware GNN and a substructure-aware deep random walk kernel to achieve dual-discriminative capability for anomalous attributes and substructures. \textbf{GmapAD} \cite{GmapAD} proposes an explainable graph mapping approach that projects graphs into a latent space for effective anomaly detection. \textbf{RQGNN} \cite{dong2023rayleigh} identifies differences in the spectral energy distributions between anomalous and normal graphs.

\paragraph{Multi-task Baselines.}  
\textbf{GraphPrompt} \cite{liu2023graphprompt} learns different task-specific prompt vectors for each task, which are added to the graph-level representations by element-wise multiplication. \textbf{All-in-One} \cite{sun2023all} treats an extra subgraph as a prompt and merges it with the original graph by cross links.

\paragraph{Hardware Specifications.} Our experiments were mainly carried out on a Linux server equipped with dual AMD EPYC 7763 64-core CPU processor, 256GB RAM, and an NVIDIA RTX 4090 GPU with 24GB memory. Some of the extremely large datasets, such as T-Finance, and certain memory-intensive baselines were implemented on the NVIDIA 8*A800 GPUs.
We mark the results as OOT (Out of Time) if the model training exceeds 2 days. For some large datasets, methods with GPU memory requirements exceeding 80GB were marked as OOM (Out of Memory), such as iGAD and GmapAD, and were attempted to be run on the CPU. iGAD was successfully completed, but GmapAD still encountered a timeout issue.

\paragraph{Metrics.} We utilize three widely used metrics to evaluate the performance of all methods: \textbf{F1-macro}, \textbf{AUROC} and \textbf{AUPRC}. \textbf{F1-macro} is the unweighted mean of the F1-scores for the two classes, disregarding the imbalance ratio between normal and anomaly labels. \textbf{AUROC} represents the area under the Receiver Operating Characteristic Curve. \textbf{AUPRC} represents the area under the Precision-Recall Curve, emphasizing model performance on imbalanced datasets by focusing on the trade-off between precision and recall.

\paragraph{Hyperparameter Tuning.}
As described in Section \ref{unimodel}, we first use an unsupervised model based on GraphMAE \cite{GraphMAE} to learn the general representation of the input features. The hyperparameters for this step are set to the default values from the official GraphMAE implementation, with 50 training epochs. Table \ref{tab:hparameters} lists all the hyperparameters used in our model along with their corresponding search spaces. During training, we conduct a \textbf{grid search} to identify the model that achieves the highest AUROC score on the validation set. Finally, we evaluate the selected model on the test sets and report the performance metrics. 
\begin{table}[ht]
    \centering
\caption{Hyperparameters search space for UniGAD. }\label{tab:hparameters}
\begin{tabular}{lll}
\toprule
& \textbf{Hyperparameter} & \textbf{Distribution} \\
\midrule
& learning rate             & Range$(5^{-4}, 10^{-2})$ \\
& activation                &  [ReLU, LeakyReLU, Tanh] \\
& hidden dimension          & [16,32,64] \\
& MRQSampler tree depth     & [1,2] \\
& GraphStitch Network layer & [1,2,3]$\times$2   \\
& epochs                    & [100, 200, 300, 400, 500] \\
\midrule
\bottomrule
\end{tabular}
\end{table}

\section{Limitations and Impacts}
\label{lim}
Since the GNN encoder we employ mainly focuses on node-level features, the learned representations may not be perfectly suited for edge and graph level tasks.
Therefore, we leave the exploration of how to integrate multiple tasks in the pre-training phase to future work.
As a generalized graph anomaly detection model, our work will be helpful in detecting classical graph anomaly applications, such as financial fraud, cybercrime, etc. 
On the other hand, an error in the recognition result may put normal groups or behaviors into anomalies, causing disturbance for the normal users in the graph network.

\section{Additional Experimental Results}
\label{addsesult}
We also provide the results of all experiments under the F1-macro and AUPRC evaluation metrics. Similar to the arrangement in the main text, for F1-macro, we show the results of multi-level performance comparison under F1-macro metric in Table \ref{tab:unifiedNE-A} and Table \ref{tab:unifiedNG-Aa}. The results of Zero-Shot Comparison under F1-macro metric are in Table \ref{transNE-A} and Table \ref{transNG-A}. For AUPRC, we show the results of multi-level performance comparison under AUPRC metric in Table \ref{tab:unifiedNE-A-PRC} and Table \ref{tab:unifiedNG-Aa-PRC}. The results of Zero-Shot Comparison under AUPRC metric are in Table \ref{transNE-A-PRC} and Table \ref{transNG-A-PRC}. It can be observed from these tables that similar conclusions can be drawn as with the AUROC results in Section \ref{exp}. UniGAD demonstrates superior performance across most datasets, regardless of whether unified or zero-shot performance is evaluated.

\begin{table}[htp]
    \caption{Comparison of unified performance (F1-macro) at both node and edge levels with different single-level methods, multi-task methods, and our proposed method.}
    \label{tab:unifiedNE-A}
    \centering
    \resizebox{\textwidth}{!}{
    \begin{tabular}{c|c|cc|cc|cc|cc|cc|cc|cc}
    \toprule
    \multirow{2}{*}{ } & \bf{Dataset} & \multicolumn{2}{c|}{\bf{Reddit}} & \multicolumn{2}{c|}{\bf{Weibo}} & \multicolumn{2}{c|}{\bf{Amazon}} & \multicolumn{2}{c|}{\bf{Yelp}} & \multicolumn{2}{c|}{\bf{Tolokers}}& \multicolumn{2}{c|}{\bf{Questions}} & \multicolumn{2}{c}{\bf{T-Finance}} \\ 
     & \bf{Task-level} & \bf{Node} & \bf{Edge}  & \bf{Node} & \bf{Edge} & \bf{Node} & \bf{Edge}  & \bf{Node} & \bf{Edge}  & \bf{Node} & \bf{Edge}  & \bf{Node} & \bf{Edge} & \bf{Node} & \bf{Edge} \\ \midrule
    \multirow{9}{*}{Node-level} 
    & GCN       & $38.81$ & / & $92.90$ & / & $62.45$ & / & $42.84$ & / & $58.64$ & / & $48.74$ & /& $70.61$ & / \\
    & GIN       & $27.22$ & / & $92.01$ & / & $72.21$ & / & $58.78$ & / & $59.40$ & / & $49.95$ & /& $76.81$ & / \\
    & GraphSAGE & $28.30$ & / & $85.05$ & / & $64.05$ & / & $64.20$ & / & $63.71$ & / & $51.07$ & / & $62.63$ & /\\
    & SGC       & $12.42$ & / & $90.49$ & / & $57.43$ & / & $47.26$ & / & $54.48$ & / & $48.59$ & /& $56.69$ & / \\
    & GAT       & $31.28$ & / & $92.46$ & / & $87.96$ & / & $57.56$ & / & $62.68$ & / & $47.28$ & / & $69.11$ & /\\
    & BernNet   & $46.27$ & / & $89.59$ & / & $89.16$ & / & $61.42$ & / & $58.73$ & / & $47.59$ & /& $70.52$ & / \\
    & PNA       & $14.08$ & / & $88.93$ & / & $61.84$ & / & $52.52$ & / & $58.49$ & / & $46.38$ & /& $27.69$ & / \\
    & AMNet     & $45.14$ & / & $89.04$ & / & $89.67$ & / & $58.75$ & / & $58.47$ & / & $49.90$ & /& $74.31$ & / \\
    & BWGNN     & $42.25$ & / & $86.94$ & / & $90.49$ & / & $64.89$ & / & $63.43$ & / & $52.09$ & / & $80.57$ & /\\ \midrule
    \multirow{9}{*}{Edge-level} 
    & GCNE   & / & \multicolumn{1}{l|}{$42.97$} & / & \multicolumn{1}{l|}{$92.32$} & / & \multicolumn{1}{l|}{$47.99$} & / & \multicolumn{1}{l|}{$42.32$} & / & \multicolumn{1}{l|}{$63.15$} & / & \multicolumn{1}{l|}{$58.40$} & / & \multicolumn{1}{l}{$79.07$} \\
    & GINE   & / & \multicolumn{1}{l|}{$36.89$} & / & \multicolumn{1}{l|}{$91.54$} & / & \multicolumn{1}{l|}{$47.82$} & / & \multicolumn{1}{l|}{$53.14$} & / & \multicolumn{1}{l|}{$60.54$} & / & \multicolumn{1}{l|}{$56.96$} & / & \multicolumn{1}{l}{$73.40$} \\
    & SAGEE  & / & \multicolumn{1}{l|}{$11.94$} & / & \multicolumn{1}{l|}{$89.59$} & / & \multicolumn{1}{l|}{$57.79$} & / & \multicolumn{1}{l|}{$57.87$} & / & \multicolumn{1}{l|}{$65.95$} & / & \multicolumn{1}{l|}{$\bf{74.15}$} & / & \multicolumn{1}{l}{$67.12$} \\
    & SGCE   & / & \multicolumn{1}{l|}{$41.08$} & / & \multicolumn{1}{l|}{$86.70$} & / & \multicolumn{1}{l|}{$55.97$} & / & \multicolumn{1}{l|}{$45.13$} & / & \multicolumn{1}{l|}{$56.50$} & / & \multicolumn{1}{l|}{$52.07$} & / & \multicolumn{1}{l}{$64.51$} \\
    & GATE   & / & \multicolumn{1}{l|}{$40.65$} & / & \multicolumn{1}{l|}{$90.53$} & / & \multicolumn{1}{l|}{$64.69$} & / & \multicolumn{1}{l|}{$49.99$} & / & \multicolumn{1}{l|}{$61.43$} & / & \multicolumn{1}{l|}{$62.80$} & / & \multicolumn{1}{l}{$68.75$} \\
    & BernE  & / & \multicolumn{1}{l|}{$39.85$} & / & \multicolumn{1}{l|}{$92.15$} & / & \multicolumn{1}{l|}{$69.12$} & / & \multicolumn{1}{l|}{$59.76$} & / & \multicolumn{1}{l|}{$61.53$} & / & \multicolumn{1}{l|}{$67.32$}  & / & \multicolumn{1}{l}{$63.16$}\\
    & PNAE   & / & \multicolumn{1}{l|}{$23.03$} & / & \multicolumn{1}{l|}{$92.30$} & / & \multicolumn{1}{l|}{$49.27$} & / & \multicolumn{1}{l|}{$52.94$} & / & \multicolumn{1}{l|}{$64.98$} & / & \multicolumn{1}{l|}{$65.39$} & / & \multicolumn{1}{l}{$65.74$} \\
    & AME    & / & \multicolumn{1}{l|}{$41.11$} & / & \multicolumn{1}{l|}{$87.04$} & / & \multicolumn{1}{l|}{$66.27$} & / & \multicolumn{1}{l|}{$57.09$} & / & \multicolumn{1}{l|}{$61.42$} & / & \multicolumn{1}{l|}{$66.74$}  & / & \multicolumn{1}{l}{$57.45$}\\
    & BWE    & / & \multicolumn{1}{l|}{$45.36$} & / & \multicolumn{1}{l|}{$91.72$} & / & \multicolumn{1}{l|}{$67.56$} & / & \multicolumn{1}{l|}{$59.30$} & / & \multicolumn{1}{l|}{$65.09$} & / & \multicolumn{1}{l|}{$66.28$} & / & \multicolumn{1}{l}{$70.88$} \\ 
     \midrule
     \multirow{2}{*}{Multi-task} & GraphPrompt-U &31.23  &38.54  & 50.64 &46.53  & 40.93 & 35.95 & 40.90 & 42.94 & 48.26  & 48.34 & 39.43 &44.61 & OOT & OOT\\
        & All-in-One-U &49.12  &2.41  &51.23  & 48.65 & 48.67 & 2.45 & 14.43	 & 46.29 &50.17   &47.90  &48.81  & 33.29 & OOT & OOT\\ 
     \midrule
    \multirow{2}{*}{\begin{tabular}[c]{@{}c@{}}UniGAD\\ (Ours)\end{tabular}} 
    & UniGAD - GCN  & \multicolumn{1}{l}{$\bf{56.70}$} & \multicolumn{1}{l|}{$\bf{53.80}$} & \multicolumn{1}{l}{$\bf{95.75}$} & \multicolumn{1}{l|}{$\bf{94.29}$} & \multicolumn{1}{l}{$69.39$} & \multicolumn{1}{l|}{$59.12$} & \multicolumn{1}{l}{$58.23$} & \multicolumn{1}{l|}{$56.76$} & \multicolumn{1}{l}{$65.20$} & \multicolumn{1}{l|}{$64.55$} & $58.06$ & $57.77$  & $84.92$ & $84.08$  \\
    &  UniGAD - BWG & \multicolumn{1}{l}{$54.08$} & \multicolumn{1}{l|}{$51.44$} & \multicolumn{1}{l}{$95.35$} & \multicolumn{1}{l|}{$94.22$} & \multicolumn{1}{l}{$\bf{91.33}$} & \multicolumn{1}{l|}{$\bf{73.59}$} & \multicolumn{1}{l}{$\bf{70.16}$} & \multicolumn{1}{l}{$\bf{63.57}$} & $\bf{68.15}$ & \multicolumn{1}{l|}{$\bf{66.20}$} & $\bf{59.45}$ & $57.54$  & $\bf{89.75}$ & $\bf{84.90}$ \\ 
    
    \bottomrule
    \end{tabular}
    }
    \end{table}

\begin{table}[htp]
        \caption{Comparison of unified performance (F1-macro) at both node and graph levels with different single-level methods, multi-task methods, and our proposed method.}
        \label{tab:unifiedNG-Aa}
        \centering
        \resizebox{\textwidth}{!}{
        \begin{tabular}{c|c|cc|cc|cc|cc|cc|cc|cc}
        \toprule
        \multirow{2}{*}{ } & \bf{Dataset} & \multicolumn{2}{c|}{\bf{BM-MN}} & \multicolumn{2}{c|}{\bf{BM-MS}} & \multicolumn{2}{c|}{\bf{BM-MT}} & \multicolumn{2}{c|}{\bf{MUTAG}} & \multicolumn{2}{c|}{\bf{MNIST0}} & \multicolumn{2}{c|}{\bf{MNIST1}}  & \multicolumn{2}{c}{\bf{T-Group}} \\ 
         & \bf{Task-level} & \bf{Node} & \bf{Graph}  & \bf{Node} & \bf{Graph} & \bf{Node} & \bf{Graph}  & \bf{Node} & \bf{Graph}  & \bf{Node} & \bf{Graph}  & \bf{Node} & \bf{Graph}  & \bf{Node} & \bf{Graph} \\ 
        \midrule
        \multirow{9}{*}{Node-level} 
        & GCN       & $68.25$ & / & $77.77$ & / & $69.72$ & / & $90.41$ & / & $92.03$  & / & $91.95$  & / & $49.50$ &  /  \\
        & GIN       & $32.96$ & / & $24.25$ & / & $25.69$ & / & $92.33$ & / & $88.88$  & / & $88.04$  & / & $49.24$ &  / \\
        & GraphSAGE & $32.96$ & / & $24.25$ & / & $25.69$ & / & $88.87$ & / & $\bf{99.99}$ & / & $\bf{99.99}$ & / & $50.77$ &  / \\
        & SGC       & $32.96$ & / & $24.33$ & / & $25.72$ & / & $54.95$ & / & $49.04$  & / & $82.70$  & / & $49.04$ &  / \\
        & GAT       & $32.96$ & / & $24.25$ & / & $25.69$ & / & $92.07$ & / & $99.94$  & / & $99.96$  & / & $50.29$ &  / \\
        & BernNet   & $35.04$ & / & $51.71$ & / & $30.26$ & / & $86.76$ & / & $\bf{99.99}$ & / & $\bf{99.99}$ & / & $52.84$ &  / \\
        & PNA       & $32.96$ & / & $24.25$ & / & $25.69$ & / & $87.49$ & / & $98.83$  & / & $99.27$  & / & $49.88$ &  / \\
        & BWGNN     & $82.48$ & / & $75.22$ & / & $76.19$ & / & $92.75$ & / & $\bf{99.99}$ & / & $\bf{99.99}$ & / & $51.81$ &  / \\ \midrule
        \multirow{6}{*}{Graph-level} 
        & OCGIN     & / & $46.15$ & / & $46.15$ & / & $46.11$  & / & $39.78$ & / & $47.79$ & / & $49.74$ & / & $48.91$   \\
        & OCGTL     & / & $46.15$ & / & $46.15$ & / & $46.15$  & / & $39.62$ & / & $47.41$ & / & $47.02$ & / & $48.91$   \\
        & GLocalKD  & / & $12.50$ & / & $12.50$ & / & $12.50$  & / & $25.59$ & / &  $8.98$ & / & $10.11$ & / &  $4.09$   \\
        & iGAD      & / & $68.29$ & / & $81.59$ & / & $89.89$  & / & $89.78$ & / & $87.73$ & / & $95.04$ & / & $46.51$   \\
        & GmapAD   & / & $46.15$ & / & $46.15$ & / & $46.15$  & / & $75.48$ & / & OOM & / & OOM & / & OOM   \\
        & RQGNN     & / & $\bf{95.46}$ & / & $93.02$ & / & $97.56$  & / & $89.39$ & / & $93.42$ & / & $96.99$ & / & $48.91$   \\
         \midrule
        \multirow{2}{*}{Multi-task} 
        & GraphPrompt-U & $36.95$ & $45.55$  & $37.46$  & $12.86$  & $47.92$ & $46.01$  & $83.08$ & $45.70$ &80.66&52.39 & 80.49	 & 28.25 & 50.77 & 49.78\\
        & All-in-One-U  & $34.98$ & $12.86$  & $21.24$  & $20.38$  & $41.36$ & $12.86$  & $38.71$ & $25.71$ & OOT & OOT & OOT & OOT & OOT & OOT\\ 
         \midrule
        \multirow{2}{*}{\begin{tabular}[c]{@{}c@{}}UniGAD\\ (Ours)\end{tabular}} 
        & UniGAD - GCN  & $\bf{99.20}$ & $83.62$ & $\bf{99.57}$ & $\bf{95.86}$ & $\bf{96.18}$ & $70.50$ & $\bf{93.33}$ & $\bf{90.00}$ & $92.17$ & $93.38$ & $92.49$  & $97.23$ & $64.98$ & $77.04$ \\
        & UniGAD - BWG  & $87.91$ & $55.89$ & $83.79$ & $61.67$ & $82.53$ & $51.36$ & $93.07$ & $89.19$ & $\bf{99.99}$ & $\bf{95.54}$ & $\bf{99.99}$ & $\bf{97.60}$ & $\bf{68.69}$ & $\bf{78.09}$ \\ 
        
        \bottomrule
        \end{tabular}
        }
    \end{table}

\begin{table}[ht!]
    \caption{Zero-shot transferability (F1-macro) at node and edge levels.}
    \label{transNE-A}
    \centering
    \resizebox{\textwidth}{!}{
    \begin{tabular}{c|cc|cc|cc|cc|cc|cc|cc}
    \toprule
    \multirow{2}{*}{ \bf{Methods} }  & \multicolumn{2}{c|}{\bf{Reddit}} & \multicolumn{2}{c|}{\bf{Weibo}} & \multicolumn{2}{c|}{\bf{Amazon}} & \multicolumn{2}{c|}{\bf{Yelp}} & \multicolumn{2}{c|}{\bf{Tolokers}} & \multicolumn{2}{c|}{\bf{Questions}} & \multicolumn{2}{c}{\bf{T-Finance}}\\ 
      & \bf{N$\rightarrow$E} & \bf{E$\rightarrow$N}  & \bf{N$\rightarrow$E} & \bf{E$\rightarrow$N} & \bf{N$\rightarrow$E} & \bf{E$\rightarrow$N}  & \bf{N$\rightarrow$E} & \bf{E$\rightarrow$N}  & \bf{N$\rightarrow$E} & \bf{E$\rightarrow$N}  & \bf{N$\rightarrow$E} & \bf{E$\rightarrow$N}  & \bf{N$\rightarrow$E} & \bf{E$\rightarrow$N}  \\ \midrule
     GraphPrompt-U & $29.87$ & $41.00$  & $50.48$  & $47.06$ & 42.19	& 33.22 &44.77& 41.44 & $47.77$ & $43.72$  & $47.63$  & $39.00$ & OOT & OOT\\
    All-in-One-U  & $2.38$  & $49.17$  & $36.82$  & $51.58$ & 12.84			 & 23.6 &12.22 & 46.04 & $33.39$ & $49.83$  & $33.08$  & $49.54$  & OOT & OOT\\ 
     \midrule
     UniGAD - GCN & \multicolumn{1}{l}{$\bf{50.61}$} & \multicolumn{1}{l|}{$\bf{50.58}$} & \multicolumn{1}{l}{$\bf{94.44}$} & \multicolumn{1}{l|}{$\bf{94.29}$} & \multicolumn{1}{l}{$56.70$} & \multicolumn{1}{l|}{$61.88$} & \multicolumn{1}{l}{$51.29$} & \multicolumn{1}{l|}{$53.36$} & \multicolumn{1}{l}{$61.14$} & \multicolumn{1}{l|}{$57.14$} & $52.01$ & $52.26$ & $80.85$ & $69.21$ \\
     UniGAD - BWG & \multicolumn{1}{l}{$49.79$} & \multicolumn{1}{l|}{$50.04$} & \multicolumn{1}{l}{$92.11$} & \multicolumn{1}{l|}{$93.62$}  & \multicolumn{1}{l}{$\bf{68.47}$} & \multicolumn{1}{l|}{$\bf{85.07}$} & \multicolumn{1}{l}{$\bf{63.69}$} & \multicolumn{1}{l|}{$\bf{71.10}$} & \multicolumn{1}{l}{$\bf{65.46}$} & $\bf{65.57}$ & \multicolumn{1}{l}{$\bf{55.81}$} & $\bf{53.44}$ & $\bf{83.35}$ & $\bf{87.61}$  \\

    \bottomrule
    \end{tabular}
    }
    \end{table}

\begin{table}[ht!]
    \caption{Zero-shot transferability (F1-macro) at node and graph levels.}
    \label{transNG-A}
    \centering
    \resizebox{\textwidth}{!}{
    \begin{tabular}{c|cc|cc|cc|cc|cc|cc}
    \toprule
    \multirow{2}{*}{  \bf{Methods} }& \multicolumn{2}{c|}{\bf{BM-MN}} & \multicolumn{2}{c|}{\bf{BM-MS}} & \multicolumn{2}{c|}{\bf{BM-MT}} & \multicolumn{2}{c|}{\bf{MUTAG}} & \multicolumn{2}{c|}{\bf{MNIST0}}  & \multicolumn{2}{c}{\bf{T-Group}} \\ 
     &  \bf{N$\rightarrow$G} & \bf{G$\rightarrow$N}  & \bf{N$\rightarrow$G} & \bf{G$\rightarrow$N} & \bf{N$\rightarrow$G} & \bf{G$\rightarrow$N}  & \bf{N$\rightarrow$G} & \bf{G$\rightarrow$N}  & \bf{N$\rightarrow$G} & \bf{G$\rightarrow$N}  & \bf{N$\rightarrow$G} & \bf{G$\rightarrow$N}   \\ \midrule
     GraphPrompt-U & $12.86$  & $34.98$ & $20.59$ & $42.78$ & $46.01$ & $43.85$ & $42.15$  & $27.73$  & 26.16		 & 26.75 & 48.47	  & 43.64\\
      All-in-One-U & $12.86$  & $34.98$ & $46.01$ & $41.25$ & $12.86$ & $22.74$ & $39.53$  & $48.80$  & OOT & OOT & OOT  & OOT\\ 
     \midrule
    UniGAD - GCN & \multicolumn{1}{l}{$\bf{53.43}$} & \multicolumn{1}{l|}{$\bf{83.88}$} & \multicolumn{1}{l}{$\bf{57.84}$} & \multicolumn{1}{l|}{$\bf{74.63}$} & \multicolumn{1}{l}{$\bf{52.78}$} & \multicolumn{1}{l|}{$\bf{54.87}$} & \multicolumn{1}{l}{$39.89$} & \multicolumn{1}{l|}{$\bf{77.47}$} & \multicolumn{1}{l}{$\bf{65.78}$} & \multicolumn{1}{l|}{$\bf{63.75}$}  & $\bf{66.47}$ & $\bf{49.22}$\\
    UniGAD - BWG & \multicolumn{1}{l}{$46.15$} & \multicolumn{1}{l|}{$59.27$} & \multicolumn{1}{l}{$46.15$} & \multicolumn{1}{l|}{$52.27$} & \multicolumn{1}{l}{$46.22$} & \multicolumn{1}{l|}{$39.54$} & \multicolumn{1}{l}{$\bf{45.58}$} & \multicolumn{1}{l|}{$66.32$} & \multicolumn{1}{l}{$44.56$} & \multicolumn{1}{l|}{$60.84$}  &\multicolumn{1}{l}{$63.92$} & $48.24$ \\ 
    \bottomrule
    \end{tabular}
    }
    \end{table}

\begin{table}[h!]
    \caption{Comparison of unified performance (AUPRC) at both node and edge levels with different single-level methods, multi-task methods, and our proposed method.}
    \label{tab:unifiedNE-A-PRC}
    \centering
    \resizebox{\textwidth}{!}{
    \begin{tabular}{c|c|cc|cc|cc|cc|cc|cc|cc}
    \toprule
    \multirow{2}{*}{ } & \bf{Dataset} & \multicolumn{2}{c|}{\bf{Reddit}} & \multicolumn{2}{c|}{\bf{Weibo}} & \multicolumn{2}{c|}{\bf{Amazon}} & \multicolumn{2}{c|}{\bf{Yelp}} & \multicolumn{2}{c|}{\bf{Tolokers}} & \multicolumn{2}{c|}{\bf{Questions}} & \multicolumn{2}{c}{\bf{T-Finance}} \\ 
     & \bf{Task-level} & \bf{Node} & \bf{Edge}  & \bf{Node} & \bf{Edge} & \bf{Node} & \bf{Edge}  & \bf{Node} & \bf{Edge}  & \bf{Node} & \bf{Edge}  & \bf{Node} & \bf{Edge} & \bf{Node} & \bf{Edge}  \\ \midrule
    \multirow{9}{*}{Node-level} 
    & GCN       & $5.84$ & / & $94.55$ & / & $36.21$ & / & $20.58$ & / & $43.68$ & / & $12.20$ & / & $70.81$ & /\\
    & GIN       & $5.89$ & / & $91.28$ & / & $75.74$ & / & $33.09$ & / & $39.71$ & / & $12.79$ & / & $61.79$ & /\\
    & GraphSAGE & $5.44$ & / & $84.97$ & / & $55.38$ & / & $45.27$ & / & $48.90$ & / & $16.72$ & / & $19.62$ & /\\
    & SGC       & $4.27$ & / & $90.70$ & / & $33.48$ & / & $16.13$ & / & $36.90$ & / & $ 9.90$ & / & $30.35$ & /\\
    & GAT       & $6.15$ & / & $90.18$ & / & $85.61$ & / & $37.51$ & / & $46.01$ & / & $15.82$ & / & $54.70$ & /\\
    & BernNet   & $7.34$ & / & $88.39$ & / & $87.08$ & / & $46.22$ & / & $43.04$ & / & $15.11$ & / & $67.65$ & /\\
    & PNA       & $5.71$ & / & $93.13$ & / & $37.28$ & / & $27.34$ & / & $42.98$ & / & $11.57$ & / & $23.07$ & /\\
    & AMNet     & $7.52$ & / & $88.59$ & / & $87.26$ & / & $46.18$ & / & $43.22$ & / & $14.39$ & / & $74.73$ & /\\
    & BWGNN     & $6.41$ & / & $92.81$ & / & $\bf{88.57}$ & / & $50.32$ & / & $49.44$ & / & $16.21$ & / & $84.94$ & /\\ \midrule
    \multirow{9}{*}{Edge-level} 
    & GCNE   & / & \multicolumn{1}{l|}{$4.68$} & / & \multicolumn{1}{l|}{$94.56$} & / & \multicolumn{1}{l|}{$16.59$} & / & \multicolumn{1}{l|}{$22.68$} & / & \multicolumn{1}{l|}{$55.15$} & / & \multicolumn{1}{l|}{$27.83$}  & / & $62.12$  \\
    & GINE   & / & \multicolumn{1}{l|}{$5.01$} & / & \multicolumn{1}{l|}{$91.84$} & / & \multicolumn{1}{l|}{$25.08$} & / & \multicolumn{1}{l|}{$28.28$} & / & \multicolumn{1}{l|}{$46.03$} & / & \multicolumn{1}{l|}{$27.34$} & / & $52.01$ \\
    & SAGEE  & / & \multicolumn{1}{l|}{$5.31$} & / & \multicolumn{1}{l|}{$91.07$} & / & \multicolumn{1}{l|}{$20.28$} & / & \multicolumn{1}{l|}{$34.22$} & / & \multicolumn{1}{l|}{$\bf{60.44}$} & / & \multicolumn{1}{l|}{$\bf{50.49}$}& / & $18.79$ \\
    & SGCE   & / & \multicolumn{1}{l|}{$3.00$} & / & \multicolumn{1}{l|}{$89.04$} & / & \multicolumn{1}{l|}{$14.54$} & / & \multicolumn{1}{l|}{$17.37$} & / & \multicolumn{1}{l|}{$51.41$} & / & \multicolumn{1}{l|}{$17.84$} & / & $30.22$\\
    & GATE   & / & \multicolumn{1}{l|}{$5.32$} & / & \multicolumn{1}{l|}{$86.61$} & / & \multicolumn{1}{l|}{$40.30$} & / & \multicolumn{1}{l|}{$31.96$} & / & \multicolumn{1}{l|}{$51.03$} & / & \multicolumn{1}{l|}{$32.70$}& / & $33.47$ \\
    & BernE  & / & \multicolumn{1}{l|}{$4.89$} & / & \multicolumn{1}{l|}{$91.34$} & / & \multicolumn{1}{l|}{$39.83$} & / & \multicolumn{1}{l|}{$32.70$} & / & \multicolumn{1}{l|}{$55.19$} & / & \multicolumn{1}{l|}{$41.52$} & / & $45.01$\\
    & PNAE   & / & \multicolumn{1}{l|}{$4.51$} & / & \multicolumn{1}{l|}{$95.24$} & / & \multicolumn{1}{l|}{$16.03$} & / & \multicolumn{1}{l|}{$28.28$} & / & \multicolumn{1}{l|}{$57.46$} & / & \multicolumn{1}{l|}{$37.61$} & / & $54.13$\\
    & AME    & / & \multicolumn{1}{l|}{$5.00$} & / & \multicolumn{1}{l|}{$87.03$} & / & \multicolumn{1}{l|}{$39.07$} & / & \multicolumn{1}{l|}{$32.46$} & / & \multicolumn{1}{l|}{$53.53$} & / & \multicolumn{1}{l|}{$40.88$} & / & $43.70$\\
    & BWE    & / & \multicolumn{1}{l|}{$5.26$} & / & \multicolumn{1}{l|}{$93.08$} & / & \multicolumn{1}{l|}{$38.83$} & / & \multicolumn{1}{l|}{$35.33$} & / & \multicolumn{1}{l|}{$58.42$} & / & \multicolumn{1}{l|}{$42.72$} & / & $68.13$\\ 
     \midrule
     \multirow{2}{*}{Multi-task} & GraphPrompt-U & $3.60$ & $2.92$ & $17.22$ & $7.31$ & $6.62$ & $2.64$ & $12.41$ & $13.63$ & $22.19$ & $33.52$ & $3.25$ & $5.22$ & OOT& OOT\\
                                  & All-in-One-U & $4.07$ & $2.93$ & $6.41$ & $5.18$ & $1.02$ & $3.13$ & $46.10$ & $13.49$ & $21.64$ & $32.16$ & $2.57$ & $4.09$ & OOT & OOT \\ 
     \midrule
    \multirow{2}{*}{\begin{tabular}[c]{@{}c@{}}UniGAD\\ (Ours)\end{tabular}} 
    & UniGAD - GCN  & \multicolumn{1}{l}{$\bf{9.73}$} & \multicolumn{1}{l|}{$\bf{5.82}$} & \multicolumn{1}{l}{$\bf{96.79}$} & \multicolumn{1}{l|}{$\bf{95.65}$} & \multicolumn{1}{l}{$38.06$} & \multicolumn{1}{l|}{$15.53$} & \multicolumn{1}{l}{$\bf{61.00}$} & \multicolumn{1}{l|}{$\bf{40.90}$} & \multicolumn{1}{l}{$46.38$} & \multicolumn{1}{l|}{$54.05$} & $15.58$ & $15.96$ & $75.30$ & $69.90$ \\
    &  UniGAD - BWG & \multicolumn{1}{l}{$5.19$} & \multicolumn{1}{l|}{$3.29$} & \multicolumn{1}{l}{$96.54$} & \multicolumn{1}{l|}{$93.66$} & \multicolumn{1}{l}{$87.28$} & \multicolumn{1}{l|}{$\bf{42.01}$} & \multicolumn{1}{l}{$27.42$} & \multicolumn{1}{l}{$24.65$} & $\bf{50.80}$ & \multicolumn{1}{l|}{$56.89$} & $\bf{17.35}$ & $19.34$ & $\bf{85.34}$ & $\bf{74.37}$\\ 
    
    \bottomrule
    \end{tabular}
    }
    \end{table}

\begin{table}[h!]
    \caption{Comparison of unified performance (AUPRC) at both node and graph levels with different single-level methods, multi-task methods, and our proposed method.}
    \label{tab:unifiedNG-Aa-PRC}
    \centering
        \resizebox{\textwidth}{!}{
        \begin{tabular}{c|c|cc|cc|cc|cc|cc|cc|cc}
        \toprule
        \multirow{2}{*}{ } & \bf{Dataset} & \multicolumn{2}{c|}{\bf{BM-MN}} & \multicolumn{2}{c|}{\bf{BM-MS}} & \multicolumn{2}{c|}{\bf{BM-MT}} & \multicolumn{2}{c|}{\bf{MUTAG}} & \multicolumn{2}{c|}{\bf{MNIST0}} & \multicolumn{2}{c|}{\bf{MNIST1}}  & \multicolumn{2}{c}{\bf{T-Group}} \\ 
         & \bf{Task-level} & \bf{Node} & \bf{Graph}  & \bf{Node} & \bf{Graph} & \bf{Node} & \bf{Graph}  & \bf{Node} & \bf{Graph}  & \bf{Node} & \bf{Graph}  & \bf{Node} & \bf{Graph}  & \bf{Node} & \bf{Graph} \\ 
        \midrule
        \multirow{9}{*}{Node-level} 
        & GCN       & $84.82$ & / & $78.23$ & / & $83.12$ & / & $82.17$ & / & $91.24$  & / & $91.29$  & / & $8.78$ &  /  \\
        & GIN       & $52.80$ & / & $32.44$ & / & $36.98$ & / & $81.91$ & / & $87.62$  & / & $87.33$  & / & $1.65$ &  / \\
        & GraphSAGE & $49.17$ & / & $32.01$ & / & $34.55$ & / & $80.20$ & / & $99.93$ & / & $99.94$ & / & $5.79$ &  / \\
        & SGC       & $51.73$ & / & $31.24$ & / & $36.42$ & / & $34.32$ & / & $82.69$  & / & $82.66$  & / & $3.93$ &  / \\
        & GAT       & $54.33$ & / & $40.83$ & / & $44.23$ & / & $82.44$ & / & $99.40$  & / & $99.90$  & / & $6.56$ &  / \\
        & BernNet   & $58.11$ & / & $38.34$ & / & $42.79$ & / & $72.17$ & / & $99.99$ & / & $99.99$ & / & $13.51$ &  / \\
        & PNA       & $72.16$ & / & $38.32$ & / & $58.97$ & / & $70.16$ & / & $98.48$  & / & $98.50$  & / & $1.03$ &  / \\
        & BWGNN     & $91.85$ & / & $70.10$ & / & $78.53$ & / & $84.33$ & / & $99.99$ & / & $99.99$ & / & $16.30$ &  / \\ \midrule
        \multirow{6}{*}{Graph-level} 
        & OCGIN     & / & $89.40$ & / & $48.80$ & / & $41.14$  & / & $31.02$ & / & $12.99$ & / & $18.09$ & / & $4.46$   \\
        & OCGTL     & / & $76.72$ & / & $46.13$ & / & $41.38$  & / & $33.87$ & / & $9.94$ & / & $11.27$ & / & $4.30$   \\
        & GLocalKD  & / & $7.71$ & / & $9.05$ & / & $17.39$  & / & $23.01$ & / &  $6.96$ & / & $13.49$ & / &  $2.51$   \\
        & iGAD      & / & $68.36$ & / & $74.57$ & / & $84.66$  & / & $91.07$ & / & $94.79$ & / & $97.98$ & / & $5.92$   \\
        & GmapAD   & / & $14.29$ & / & $14.29$ & / & $14.29$ & / & $60.96$ & / & OOM & / & OOM & / & OOM   \\
        & RQGNN     & / & $\bf{99.32}$ & / & $97.60$ & / & $99.36$  & / & $91.27$ & / & $97.62$ & / & $98.39$ & / & $7.98$   \\
         \midrule
        \multirow{2}{*}{Multi-task} 
        & GraphPrompt-U & $43.87$ & $15.15$  & $26.15$  & $14.76$  & $27.78$ & $14.83$  & $70.41$ & $60.70$ & $82.89$ & $36.25$ & $83.30$ & $5.97$ & $1.06$ & $4.25$\\
        & All-in-One-U  & $57.75$ & $8.58$  & $35.75$  & $9.16$  & $25.13$ & $20.23$  & $5.86$ & $33.09$ & OOT & OOT & OOT & OOT & OOT & OOT
\\ 
         \midrule
        \multirow{2}{*}{\begin{tabular}[c]{@{}c@{}}UniGAD\\ (Ours)\end{tabular}} 
        & UniGAD - GCN  & $\bf{99.63}$ & $73.54$ & $\bf{99.91}$ & $\bf{98.39}$ & $\bf{99.73}$ & $\bf{99.99}$ & $86.60$ & $91.66$ & $95.94$ & $94.86$ & $96.38$  & $98.36$ & $21.53$ & $44.95$ \\
        & UniGAD - BWG  & $91.19$ & $23.83$ & $85.81$ & $30.89$ & $84.93$ & $14.74$ & $\bf{87.15}$ & $\bf{92.00}$ & $\bf{99.99}$ & $\bf{97.92}$ & $\bf{99.99}$ & $\bf{98.60}$ & $\bf{31.31}$ & $\bf{55.64}$  \\ 
        
        \bottomrule
        \end{tabular}
        }
    \end{table}

\begin{table}[h!]
    \caption{Zero-shot transferability (AUPRC) at node and edge levels.}
    \label{transNE-A-PRC}
    \centering
    \resizebox{\textwidth}{!}{
    \begin{tabular}{c|cc|cc|cc|cc|cc|cc|cc}
    \toprule
    \multirow{2}{*}{ \bf{Methods} }  & \multicolumn{2}{c|}{\bf{Reddit}} & \multicolumn{2}{c|}{\bf{Weibo}} & \multicolumn{2}{c|}{\bf{Amazon}} & \multicolumn{2}{c|}{\bf{Yelp}} & \multicolumn{2}{c|}{\bf{Tolokers}} & \multicolumn{2}{c|}{\bf{Questions}} & \multicolumn{2}{c}{\bf{T-Finance}} \\ 
      & \bf{N$\rightarrow$E} & \bf{E$\rightarrow$N}  & \bf{N$\rightarrow$E} & \bf{E$\rightarrow$N} & \bf{N$\rightarrow$E} & \bf{E$\rightarrow$N}  & \bf{N$\rightarrow$E} & \bf{E$\rightarrow$N}  & \bf{N$\rightarrow$E} & \bf{E$\rightarrow$N}  & \bf{N$\rightarrow$E} & \bf{E$\rightarrow$N}  & \bf{N$\rightarrow$E} & \bf{E$\rightarrow$N} \\ \midrule
     GraphPrompt-U & $2.80$ & $3.00$ & $7.33$ & $16.50$ & $2.37$ & $9.83$ & $13.69$ & $13.26$ & $34.14$ & $21.59$ & $5.49$ & $3.57$ &OOT &OOT \\
     All-in-One-U  & $2.99$ & $4.15$ & $5.37$ & $6.53$ & $10.20$ & $3.13$ & $14.43$ & $13.49$ & $31.77$ & $21.79$ & $4.11$ & $3.04$ &OOT &OOT \\ 
     \midrule
     UniGAD - GCN & \multicolumn{1}{l}{$\bf{3.89}$} & \multicolumn{1}{l|}{$\bf{5.44}$} & \multicolumn{1}{l}{$\bf{93.35}$} & \multicolumn{1}{l|}{$\bf{95.57}$} & \multicolumn{1}{l}{$10.86$} & \multicolumn{1}{l|}{$29.44$} & \multicolumn{1}{l}{$22.00$} & \multicolumn{1}{l|}{$24.98$} & \multicolumn{1}{l}{$51.50$} & \multicolumn{1}{l|}{$40.51$} & $12.44$ & $\bf{6.76}$ & $69.67$ & $69.74$ \\
    UniGAD - BWG & \multicolumn{1}{l}{$3.11$} & \multicolumn{1}{l|}{$4.09$} & \multicolumn{1}{l}{$86.47$} & \multicolumn{1}{l|}{$93.21$} & \multicolumn{1}{l}{$\bf{28.07}$} & \multicolumn{1}{l|}{$\bf{78.56}$} & \multicolumn{1}{l}{$\bf{35.99}$} & \multicolumn{1}{l|}{$\bf{55.04}$} & \multicolumn{1}{l}{$\bf{54.26}$} & $\bf{46.36}$ & \multicolumn{1}{l}{$\bf{13.89}$} & $5.80$ & $\bf{70.61}$ & $\bf{81.02}$ \\

    \bottomrule
    \end{tabular}
    }
\end{table}
\begin{table}[h]
    \caption{Zero-shot transferability (AUPRC) at node and graph levels.}
    \label{transNG-A-PRC}
    \centering
    \resizebox{\textwidth}{!}{
    \begin{tabular}{c|cc|cc|cc|cc|cc|cc}
    \toprule
    \multirow{2}{*}{  \bf{Methods} }& \multicolumn{2}{c|}{\bf{BM-MN}} & \multicolumn{2}{c|}{\bf{BM-MS}} & \multicolumn{2}{c|}{\bf{BM-MT}} & \multicolumn{2}{c|}{\bf{MUTAG}}   & \multicolumn{2}{c|}{\bf{MNIST1}}  & \multicolumn{2}{c}{\bf{T-Group}} \\ 
     &  \bf{N$\rightarrow$G} & \bf{G$\rightarrow$N}  & \bf{N$\rightarrow$G} & \bf{G$\rightarrow$N} & \bf{N$\rightarrow$G} & \bf{G$\rightarrow$N}  & \bf{N$\rightarrow$G} & \bf{G$\rightarrow$N}  & \bf{N$\rightarrow$G} & \bf{G$\rightarrow$N}  & \bf{N$\rightarrow$G} & \bf{G$\rightarrow$N}   \\ \midrule
     GraphPrompt-U & $13.87$  & $48.25$ & $14.51$ & $26.88$ & $13.17$ & $26.90$ & $\bf{63.44}$  & $46.84$   & $5.95$  & $22.23$&  $5.04$	& $1.15$ \\
      All-in-One-U & $\bf{78.62}$  & $73.54$ & $9.24$ & $32.03$ & $11.94$ & $24.38$ & $28.38$  & $9.60$ & OOT & OOT & OOT  & OOT\\ 
     \midrule
    UniGAD - GCN & \multicolumn{1}{l}{$34.05$} & \multicolumn{1}{l|}{$\bf{86.00}$} & \multicolumn{1}{l}{$\bf{42.41}$} & \multicolumn{1}{l|}{$\bf{81.74}$} & \multicolumn{1}{l}{$\bf{24.26}$} & \multicolumn{1}{l|}{$\bf{60.50}$} & \multicolumn{1}{l}{$40.63$} & \multicolumn{1}{l|}{$\bf{52.67}$}  & $\bf{9.22}$ & $25.59$ & $\bf{35.02}$ & $\bf{6.97}$\\
    UniGAD - BWG & \multicolumn{1}{l}{$21.25$} & \multicolumn{1}{l|}{$54.09$} & \multicolumn{1}{l}{$26.62$} & \multicolumn{1}{l|}{$31.35$} & \multicolumn{1}{l}{$16.86$} & \multicolumn{1}{l|}{$38.62$} & \multicolumn{1}{l}{$38.64$} & \multicolumn{1}{l|}{$27.56$} & \multicolumn{1}{l}{$7.84$} & $\bf{32.94}$ &\multicolumn{1}{l}{$27.15$} & $4.30$  \\ 
    \bottomrule
    \end{tabular}
    }
    \end{table}

\end{document}